\def\paperversion{2}

\ifnum\paperversion=1
\documentclass{IMAIAI}
\fi

\ifnum\paperversion=2
\documentclass[11pt]{article}
\usepackage{url}
\usepackage{smile}
\usepackage[normalem]{ulem}
\usepackage{kpfonts}
\usepackage[colorlinks, linkcolor=blue, anchorcolor=blue, citecolor=blue]{hyperref}
\usepackage[margin=1in]{geometry}
\usepackage{cite}
\fi

\usepackage{amsmath}%
\usepackage{amsthm}%
\usepackage{soul}

\usepackage{graphics}%
\usepackage{algorithmic}%
\usepackage[table]{xcolor}
\usepackage{booktabs}
\usepackage{caption}
\ifnum\paperversion=1
\usepackage{subcaption}
\usepackage{makecell}
\usepackage{dsfont}
\newcommand{\RR}{\mathbb{R}}
\newcommand{\cF}{\mathcal{F}}
\fi
\usepackage{pgfplots}
\usepackage{natbib}
\newtheorem*{theorem*}{Theorem}

\theoremstyle{plain}
\newcommand{\W}{\mathcal{W}}

\newcommand{\diff}{\mathrm{d}}
\newcommand{\cv}[1]{\mathbf{#1}}
\newcommand{\mbE}{\mathbb{E}}

\newcommand{\dsI}{\mathds{1}}

\newcommand{\trans}{^{\mathrm T}}

\newcommand{\FNN}{\mathcal{F}_{\mathrm{NN}}}
\newcommand{\ReLU}{\mathrm{ReLU}}
\newcommand{\HU}{\mathcal{H}^{s,\beta}(\mathcal{U})}
\newcommand{\inp}[2]{\langle#1,#2\rangle}

\newcommand{\tI}{\text{(I)}}
\newcommand{\tII}{\text{(II)}}

\definecolor{longhorn}{rgb}{0.8, 0.33, 0.0}

\ifnum\paperversion=1
\newtheorem{assumption}{Assumption}
\newcommand{\cU}{\mathcal{U}}
\newcommand{\cT}{\mathcal{T}}
 
\newcommand{\cH}{\mathcal{H}}
\newcommand{\cA}{\mathcal{A}}
\newcommand{\cW}{\mathcal{W}}
\fi

\usepackage{pifont}

\begin{document}

\title{
A Manifold Two-Sample Test Study: Integral Probability Metric with Neural Networks
}

\ifnum\paperversion=1
\shorttitle{Integral Probability Metric with Neural Networks} %
\shortauthorlist{Jie, Minshuo, Tuo, Wenjing, Yao} %
\author{{%
\sc Jie Wang}
\\[2pt]
School of ISyE, Georgia Institute of Technology\\
Atlanta, GA 30332
\\[2pt]
{\sc Minshuo Chen}\\[2pt]
School of ISyE, Georgia Institute of Technology\\
{\sc Tuo Zhao}\\[2pt]
School of ISyE, Georgia Institute of Technology\\
{\sc Wenjing Liao}\\[2pt]
School of Math, Georgia Institute of Technology\\
{\sc and}\\[2pt]
{\sc Yao Xie} \\[2pt]
School of ISyE, Georgia Institute of Technology
}
\fi

\ifnum\paperversion=2
\author{Jie Wang, Minshuo Chen, Tuo Zhao, Wenjing Liao, Yao Xie
\thanks{Jie Wang, Minshuo Chen, Tuo Zhao, and Yao Xie are affiliated with School of Industrial
and Systems Engineering at Georgia Tech;
Wenjing Liao is affiliated with School of Mathematics at Georgia Tech;
Email: \{jwang3163, mchen393, tourzhao, wliao60\}@gatech.edu, yao.xie@isye.gatech.edu
}
}
\fi

\maketitle

\begin{abstract}
{
Two-sample tests are important areas aiming to determine whether two collections of observations follow the same distribution or not.
We propose two-sample tests based on integral probability metric~(IPM) for high-dimensional samples supported on a low-dimensional manifold.
We characterize the properties of proposed tests with respect to the number of samples $n$ and the structure of the manifold with intrinsic dimension $d$.
When an atlas is given, we propose a two-step test to identify the difference between general distributions, which achieves the type-II risk in the order of $n^{-1/\max\{d,2\}}$.
When an atlas is not given, we propose H{\"o}lder IPM test that applies for data distributions with $(s,\beta)$-H{\"o}lder densities, which achieves the type-II risk in the order of $n^{-(s+\beta)/d}$.
To mitigate the heavy computation burden of evaluating the H{\"o}lder IPM, we approximate the H{\"o}lder function class using neural networks.
Based on the approximation theory of neural networks, 
we show that the neural network IPM test has the type-II risk in the order of $n^{-(s+\beta)/d}$, which is in the same order of the type-II risk as the H{\"o}lder IPM test.
Our proposed tests are adaptive to low-dimensional geometric structure because their performance crucially depends on the intrinsic dimension instead of the data dimension.
}
{
\ifnum\paperversion=2
\\\textit{Keywords:}
\fi
H{\"o}lder densities, Deep neural networks, Two-sample testing
}
\end{abstract}

\section{Introduction}
As an important topic in statistical inference, two-sample testing aims to determine whether two sets of collected samples are from the same distribution or not.
Such a topic has wide applications in general scientific discovery areas.
For example, the adversarial attack is a popular research topic in machine learning: imperceptible perturbations to testing data can lead to misbehavior of the trained machine learning model, such as wrong predictions~\citep{goodfellow2014explaining}. 
One can perform two-sample testing based on a group of normal samples and a group of testing samples to determine whether testing samples are anomalous for \emph{adversarial detection}~\citep{sheikholeslami2021provably, gu2019detecting, tramer2022detecting}.
Other notable examples in which two-sample test plays a key role include \emph{change-point detection}~\citep{Xie_2013,cao2018change,xie2021sequential}, \emph{model criticism}~\citep{Lloyd15, chwialkowski2016kernel, bikowski2018demystifying}, \emph{causal inference}~\citep{lopezpaz2018revisiting}, etc.

The problem of two-sample testing has been a long-standing challenge in statistics.
Classical approaches {(see, e.g., \cite{lehmann2005testing})} for two-sample tests mainly focus on parametric or low-dimensional settings, such as the Hotelling's two-sample test~\citep{hotelling1931} and the Student's t-test~\citep{PFANZAGL96}.
Non-parametric two-sample testing for high-dimensional data has recently received much attention, {in which no prior information about the data-generating distributions is available.}
Existing non-parametric tests are usually constructed based on certain distance functions quantifying the discrepancy between probability distributions.
Many existing tests leverage \emph{integral probability metric}~(IPM, \citet{muller1997integral}) as such distance functions, which {operates} by finding a critic within a certain function space that {maximally distinguishes two distributions}.
Commonly used IPMs in literature include the total variation distance~\citep{Ga1991}, the Wasserstein distance~\citep{delbarrio1999, ramdas2015wasserstein}, and the Maximum Mean Discrepancy~(MMD, \citet{Gretton12, Gretton09, Grettonnips12}).

Despite that non-parametric tests for high-dimensional data has promising applications, their performances usually degrade quickly as the data dimension increases, known as the \emph{curse of dimensionality} issue.
For example, the Kolmogorov-Smirnov test~\citep{Frank51,Pratt1981} and Anderson-Darling test~\citep{anderson1952asymptotic} {are powerful for univariate samples}, but their computation is not tractable for multivariate data~(see, e.g., \cite{justel1997multivariate}).
The Hotelling's $T^2$ test~\citep{anderson1962introduction,muirhead2009aspects} is a classical approach {for distinguishing mean and variance differences in multivariate data,} but it performs poorly when dealing with $n$ samples of dimension $D$ for the case $D/n\to\gamma\in(0,1)$~\citep{bai1996effect}, and it is even undefined for the case $D>n$.
The Wasserstein test~\citep{ramdas2015wasserstein} is computationally tractable, but it is shown to suffer from the curse of dimensionality~\citep{fournier2015rate, weed2019estimation}.
MMD test is a notable contribution in literature, but for properly selected kernel functions, its power {decays polynomially or even exponentially} into zero with an increasing dimension of data distributions~\citep{reddi2014decreasing}, which indicates it is not suitable for high-dimensional testing as well.

To tackle the difficulty of high-dimensional non-parametric tests, we take the low-dimensional geometric structure of data points into consideration.
This consideration is not restrictive in practice, as real-world high-dimensional data usually {exhibit} low-dimensional structure.
For example, images can be viewed as points in high-dimensional space, with coordinates representing the intensity of pixels, but their perceptually meaningful structure usually depends on a small number of parameters, {due to} rotation, translation, and skeletonization~\citep{hinton2006reducing, osher2017low, gong2019intrinsic}.
In Table~\ref{Tab:dimension:sum}, we showcase the (estimated) intrinsic dimensions of popular real data sets and commonly encountered analytical manifolds in optimization and simulation literature. 
As can be seen, in real data sets (MNIST, SVHN, CIFAR-10, ImageNet), the data dimension is much larger than the estimated intrinsic dimension.
The degree of freedom of data points from general areas (e.g., visual, acoustic, textual) is often significantly smaller than the input data dimension due to rich local regularities, global symmetries, or repetitive patterns~\citep{tenenbaum2000global, roweis2000nonlinear, Djuric15}.
It is, therefore, reasonable to assume that the data-generating distributions are {supported on} the manifold with a small intrinsic dimension when designing two-sample tests.

 \begin{table}[t]
	\vspace{-2px}
	\caption{
	Estimated intrinsic dimensions of popular data sets in machine learning and analytical manifolds}
	\label{Tab:dimension:sum}
	\footnotesize
	\begin{center}
			\begin{tabular}{lcc}
			\toprule
				Data Set or Manifold  & Data Dimension & (Estimated) Intrinsic Dimension \\
				\midrule
MNIST~\citep{lecun1998gradient} & 784 & between $7$ and $13$~\citep{pope2021the}\\
SVHN~\citep{sermanet2012convolutional} & 3072 & between $9$ and $19$~\citep{pope2021the}\\
CIFAR-10~\citep{abouelnaga2016cifar} & 3072 & between $11$ and $23$~\citep{pope2021the}\\
		ImageNet~\citep{deng2009imagenet} & 150528 & between $26$ and $43$~\citep{pope2021the}\\
	    Sphere $\{x\in\mathbb{R}^{D}:~\|x\|_2=1\}$	 & $D$ & $D-1$\\
		Stiefel manifold $\{X\in\mathbb{R}^{m\times k}:~X\trans X=I_k\}$	  & $km$ & $km-\frac{1}{2}k(k+1)$ \\
		Oblique manifold $\{X\in\mathbb{R}^{m\times k}:~(X\trans X)_{i,i}=1, \forall i\}$ & $km$ &  $k(m-1)$\\
		Fixed-rank manifold $\{X\in\mathbb{R}^{m\times k}:~\text{rank}(X)=r\}$	  & $km$ & $(m+k-r)r$ \\
				\bottomrule
			\end{tabular}
	\end{center}
	\end{table}
	
It is well-known that neural networks are powerful in function representation and are adaptive to the structure of data.
Classical works have established universal approximation theories of neural networks, in which the network size depends crucially on the data dimension $D$.
However, it is worth mentioning that those works mentioned above do not fully explain the success of neural networks since the data dimension $D$ is usually a large number, so the sample complexity rate appears to be too slow.
Recent works provide theoretical justification when the low-dimensional data structure is in force, indicating the sample complexity rate usually depends on the intrinsic dimension instead of the data dimension~\citep{chen2019efficient, chen2020nonparametric, Schmidt_Hieber_2020}.
Our proposed manifold two-sample test is built upon this recent progress on neural network-based learning with the low-dimensional data structure.

In this work, we design an IPM test for manifold data and investigate its statistical performance guarantees, where the critic, a proxy to learn the structure of data, is parameterized using neural networks.
Our theory uncovers that the type-II risk of neural network-based two-sample test on manifold data depends on the intrinsic dimension (and only weakly depends on the data dimension), indicating it does not have a curse of dimensionality issue.
The main results are summarized in the following subsection.

\subsection{Summary of main results}

 \begin{table}[t]
	\vspace{-2px}
	\caption{Summary of our proposed manifold two-sample tests}
	\label{fig:summary:tests}
	\footnotesize
	\begin{center}
			\begin{tabular}{lccccccccc}
			\toprule
				Algorithm  & Atlas? & Target Distributions & Rate of Type-II Risk\\
				\midrule
			 Two-step Test (Section~\ref{Sec:2:A}) & Given & Arbitrary in $\cU$ & $\tilde{O}\left(n^{-1/\max\{d,2\}}\right)$\\
				H{\"o}lder IPM Test (Section~\ref{Sec:test:structured}) & Not given & Smooth densities in $\cH^{s,\beta}(\cU)$ & $O\left(n^{-(s+\beta)/d}\right)$\\
				NN IPM Test (Section~\ref{Sec:estimate:IPM:NN}) & Not given & Smooth densities in $\cH^{s,\beta}(\cU)$ & $O\left(n^{-(s+\beta)/d}\right)$\\
				\bottomrule
			\end{tabular}
	\end{center}
	\end{table}

{Suppose we have two collections of data points $\cv x^n:=\{\cv x_1,\ldots, \cv x_n\}$ and $\cv y^n:=\{\cv y_1,\ldots, \cv y_n\}$ from two distributions $p$ and $q$ respectively.
The two distributions are supported on a $d$-dimensional compact Riemannian manifold $\cU$ isometrically embedded in $\RR^D$ with $d \ll D$.}
{We consider for simplicity that the sample sizes of two collections are the same, yet our analysis can be extended to imbalanced data.}
The goal is to design a test which, given samples $\cv x^n$ and $\cv y^n$, decides whether to accept the null hypothesis $H_0:~p=q$ or the alternative hypothesis $H_1:~p\ne q$.
Given $0<\eta,t<1/2$, we aim to {ensure that} the type-I risk {is} at most $\eta$ (which we call at level $\eta$), and the type-II risk {is} at most $t$ (of power $1-t$); and we aim to {make the specification of $t$ as small as possible while keeping $\eta$ as a fixed number.}

In the following, we present the summary of our results on two-sample testing for manifold data (see Table~\ref{fig:summary:tests}).

\paragraph{Two-step Test}
{
When the structure of the manifold $\cU$, i.e., an atlas describing a collection of local neighborhoods of $\cU$, is exactly known, we design a two-sample test for general distributions supported on $\cU$.}
{Given an atlas, a distribution on $\cU$ can be decomposed as a convex combination of conditional distributions supported on each local neighborhood.
In other words, a distribution can be viewed as a finite mixture distribution with such a decomposition.
We thereby propose a two-step test to examine the difference between two mixture distributions.}
{In the first step}, we examine whether the weight coefficients between them are the same or not using the $\mathcal{L}_2$ divergence as testing statistic.
{In the second step, we examine whether all the conditional distributions are the same. We use a special projected Wasserstein distance~\citep{wang2020twosample, wang2020twosamplekernel}, i.e., the Wasserstein distance between projected probability distributions with a special low-dimensional mapping, to extract the difference between conditional distributions.
We show that the type-II risk of the proposed test with a fixed type-I risk is $\tilde{O}\left( 
n^{-1/\max\{d,2\}}
\right)$.
}

\paragraph{H{\"o}lder IPM Test}
When an atlas is not given, we propose H{\"o}lder IPM test that works for distributions with H{\"o}lder densities with order $(s,\beta)$ on $\cU$.
The null hypothesis $H_0$ is rejected when the H{\"o}lder IPM between two empirical distributions is greater than a certain threshold.
We show that the type-II risk of H{\"o}lder IPM test with a fixed type-I risk is $O\left( 
n^{-\frac{s+\beta}{d}}
\right)$.

\paragraph{NN IPM Test}
{
To lift the heavy computational burden of evaluating the H{\"o}lder IPM, we approximate the H{\"o}lder function class using neural networks, in leveraging the approximation power of neural networks.}
We focus on feedforward neural networks~(FNN) with {entrywise} ReLU activation function, i.e., $\ReLU(x)=\max\{0,x\}$ to {parameterize} the H{\"o}lder IPM.
{Neural} networks are widely used in machine learning {applications}~\citep{nair2010rectified,glorot2011deep,maas2013rectifier} such as computer vision, speech recognition, and natural language processing.
Given an input $\cv x\in\mathbb{R}^D$, an $L$-{layer} FNN returns {an output}
\begin{equation}\label{Eq:NN:function}
f(\cv x)=W_L\cdot\ReLU(W_{L-1}\cdots\ReLU(W_1\cv x+\cv b_1)\cdots+\cv b_{L-1})+\cv b_L,
\end{equation}
where $W_1,\ldots,W_L$, and $\cv b_1,\ldots, \cv b_L$ are weight matrices and vectors of proper sizes, respectively.
Denote by $\FNN(R,\kappa,L,t,K)$ the collection of neural networks with bounded weight parameters and bounded output:
\begin{equation}\label{Eq:specific:function:class}
\begin{aligned}
&\FNN(R,\kappa,L,t,K)=\bigg\{
{f~\bigg|~
\mbox{$f(\cv x)$ has the form $\eqref{Eq:NN:function}$ with $L$-layers and width bounded by $t$},}
\\
&\qquad\qquad\qquad\|f\|_{\infty}\le R,
\|W_i\|_{\infty,\infty}\le\kappa,
\|\cv b_i\|_{\infty}\le\kappa, i=1,2,\ldots,L,
\sum_{i=1}^L\|W_i\|_0+\|\cv b_i\|_0\le K
\bigg\},
\end{aligned}
\end{equation}
where $\|\cdot\|_0$ denotes the number of non-zero entries of the argument, $\|\cdot\|_{\infty}$ denotes the maximum {magnitude} of all entries of the argument.
Given a matrix $H$, denote $\|H\|_{\infty,\infty}=\max_{i,j}~|H_{i,j}|$.
Our results {are} summarized as follows.

\begin{theorem*}[Informal version of Theorem~\ref{Theorem:test:property:app}]
{
Let $\cU$ be a $d$-dimensional compact Riemannian manifold isometrically embedded in $\RR^D$ with mild regularity conditions, and $d\ll D$.
Consider two target distributions $p, q$ supported on $\cU$ with H{\"o}lder $(s,\beta)$-density functions.
When the hyper-parameters of the network function class $\FNN(R,\kappa,L,t,K)$ defined in \eqref{Eq:specific:function:class} is chosen as
\[
\begin{aligned}
R&=1,\quad \kappa=O(\sqrt{d}),\quad L=O\left(
\frac{s+\beta}{2(s+\beta)+d}\log(nD)
\right),\quad t={O}\left(
n^{\frac{d}{2(s+\beta)+d}}+D
\right),\quad\text{and }\\
K&={O}\left(
\frac{s+\beta}{2(s+\beta)+d}n^{\frac{d}{2(s+\beta)+d}}\log n
+\frac{s+\beta}{2(s+\beta)+d}D\log n + D\log D
\right),
\end{aligned}
\]
the type-II risk of the NN IPM test with a fixed type-I risk is of $O\left( 
n^{-\frac{s+\beta}{d}}
\right)$.
}
\end{theorem*}

Our proposed test requires the network size to depend on the intrinsic data dimension $d$ and weakly depend on the input data dimension $D$ to maintain satisfactory testing power.
We show that the NN IPM test has the type-II risk in the order of $n^{-(s+\beta)/d}$, which is in the same order of the type-II risk as the H{\"o}lder IPM test.
We remark that the order of type-II risk only depends on the intrinsic dimension $d$, which is usually much smaller than that with the data dimension $D$.
It, therefore, demonstrates the power of neural networks for tackling the curse of dimensionality challenge of non-parametric tests for data with low-dimensional geometric structure.

\subsection{Related work}

Many existing works have studied manifold two-sample tests.
A series of works leverage kernel functions to learn the structure of data to design tests~\citep{wang2020twosample, xie2021sequential, mueller2015principal, lin2020projection, lin2020projection2, wang2020twosamplekernel}.
The performance of their proposed tests is sensitive to the choice of kernel functions. 
From a theoretical perspective, \citet{cheng2021kernel} show that the performance of MMD tests depends on the intrinsic dimension provided that the bandwidth is properly selected.
The theoretical results of their work are based on an isotropic kernel choice, which may not work well in practical scenarios.

Neural networks in the hypothesis testing area have achieved some success~\citep{liu2020learning, cheng2020classification, cheng2021neural}.
In particular, \cite{liu2020learning} parameterizes the MMD kernel function using neural networks to maximize the testing power.
\cite{cheng2021neural} applies the lazy training technique of neural networks to speed up the MMD tests.
A closely related work~\citep{cheng2020classification} leverages neural network-based classifier tests with performance guarantees under training and testing data split, which is a different framework. Specifically, it was shown that both the neural network complexity and training sampling complexity only depend on the intrinsic dimension rather than the data ambient dimension, and the two-sample test efficiency on the test split is shown by an asymptotic result. In this work, our analysis does not use training and testing split and gives a non-asymptotic characterization of testing power, which reveals the impacts of intrinsic dimension, sample size, and smoothness parameters.

\begin{table}[t!]\footnotesize
	\centering
	\begin{tabular}{c|l||c|l}
		\hline\hline
		Notation & Description & Notation & Description \\[2pt]
		\hline
		$\cU$     & Manifold		& $\{U_{\alpha},\phi_{\alpha}\}$ & Atlas for $\cU$\\[3pt]\hline
		$\{\rho_{\alpha}\}$ & Partition of Unity & $d_{\mathcal{F}}$ & IPM with function class $\cF$\\[3pt]\hline
		$\cH^{s,\beta}(\cU)$ & H{\"o}lder function class on $\cU$ & $\FNN$ & Neural network class\\[3pt]\hline
		$\{\cv x_i,\cv y_i\}_{i=1}^n$ & Given data set & $\ell_n$ & A threshold for two-sample test
		\\[3pt]\hline
		$\|\cv p-\cv q\|_2$ & 
		\makecell[l]{$\mathcal{L}_2$ divergence between discrete\\ distributions $\cv p$ and $\cv q$}
		&$T(p,q)$ & 
		\makecell[l]{maximum of projected Wasserstein\\ distances between distributions $p,q$ on $\cU$}
		\\[6pt]\hline
		$\hat{p},\hat{q}$ & \makecell[l]{Empirical distributions of\\\ $p,q$ from data}
		&$\psi_{\#}p$ & \makecell[l]{Pushforward of distribution\\$p$ under mapping $\psi$}
		\\[6pt]
				\hline\hline
	\end{tabular}
	\caption{Notations used in this paper. }\label{tab:notation}
\end{table}

\subsection{Roadmap and notations}
The rest of the paper is organized as follows: 
Section~\ref{sec:preliminariy} presents a brief introduction to manifolds and distance measures between distributions.
Section~\ref{Sec:main:results} presents a theoretical analysis of two-sample tests for manifold data.
Section~\ref{Sec:proof:Thm:agn:testing} sketches the proof of the type-II risk for a two-step test with a given atlas.
Section~\ref{Thm:test:property:proof} sketches the proof of the type-II risk for H{\"o}lder IPM test.
Section~\ref{Sec:proof:Theorem:test:property:app} sketches the proof of the type-II risk for the NN IPM test that approximates H{\"o}lder IPM test efficiently.
Section~\ref{Sec:conclusion} provides the conclusion of the paper.

We use bold-faced letters to denote vectors, and normal font letters with a subscript to denote its coordinate, i.e., $\cv x\in\mathbb{R}^d$ and $x_k$ being the $k$-th coordinate of $\cv x$.
Given a vector $\cv n=[n_1,\ldots,n_d]\trans\in\mathbb{N}^d$, we define $\cv n!=\prod_{i=1}^d(n_i)!$ and $|\cv n|=\sum_{i=1}^dn_i$.
We define $\cv x^n=\prod_{i=1}^dx_i^{n_i}$.
Given a function $f:~\mathbb{R}^d\to\mathbb{R}$, denote its derivative as $D^{\cv n}f = \frac{\partial^nf}{\partial x_1^{n_1}\cdots\partial x_d^{n_d}}$.
We say a function $f:~\mathbb{R}^d\to\mathbb{R}$ is $\mathcal{C}^s$ if it is continuously differentiable up to order $s$.
We use $\circ$ to represent the function composition operator.
Given a measurable space $(\Omega,\mathcal{F})$, we say the function $p:~\mathcal{F}\to\mathbb{R}$ is a probability distribution if $p$ satisfies measure properties and $p(\Omega)=1$.
With a probability distribution $p$ and an event $\mathcal{E}$, we denote $p\mid\mathcal{E}$ as the distribution $p$ conditioned on the event $\mathcal{E}$.
More notations used in this paper is summarized in Table~\ref{tab:notation}.

\section{Preliminaries}\label{sec:preliminariy}
In this section, we first review some basic concepts about manifolds, and the details can be found in \cite{tu2010introduction}, and \cite{lee2006riemannian}.
Then we define different types of integral probability metrics considered throughout this paper. 

\subsection{Low-dimensional manifold}
Denote by $\mathcal{U}$ a $d$-dimensional Riemannian manifold \emph{isometrically embedded} in $\mathbb{R}^{D}$, where we assume that $d\ll D$.
For instance, a circle is a $1$-dimensional manifold embedded in $\mathbb{R}^2$.
\begin{definition}[Chart]
A chart for $\mathcal{U}$ is represented as a pair $(U,\phi)$ such that $U\subseteq\mathcal{U}$ is an open set, and $\phi:~U\to \mathbb{R}^d$ is a \emph{homeomorphism}.
\end{definition}

A chart essentially defines a local coordinate system from $\mathcal{U}$ to $\mathbb{R}^d$.
Let $(U,\phi)$ and $(V,\psi)$ be two charts on $\cU$.
We say they are $\mathcal{C}^k$-compatible if the following two transition functions are both $\mathcal{C}^k$:
\[
\phi\circ\psi^{-1}:~\psi(U\cap V)\to \phi(U\cap V),\quad
\psi\circ\phi^{-1}:~\phi(U\cap V)\to \psi(U\cap V).
\]
We now present the definition of an atlas.
\begin{definition}[$\mathcal{C}^k$ Atlas]
An atlas for $\mathcal{U}$ is denoted as $\{(U_{\alpha}, \phi_{\alpha})\}_{\alpha\in\mathcal{A}}$ of pairwise $\mathcal{C}^k$-compartible charts such that $\bigcup_{\alpha\in\mathcal{A}}U_{\alpha}=\mathcal{U}$.
An atlas is finite if it contains finitely many charts.
\end{definition}
The existence of an atlas on $\mathcal{U}$ allows us to define $\mathcal{C}^s$ functions on $\mathcal{U}$.
\begin{definition}[$\mathcal{C}^s$ Functions on $\mathcal{U}$]
Consider a smooth manifold $\mathcal{U}$ in $\mathbb{R}^D$.
A function $f:~\mathcal{U}\to\mathbb{R}$ is said to be $\mathcal{C}^s$ if for any chart $(U,\phi)$, the composition $f\circ\phi^{-1}:~\phi(U)\to\mathbb{R}$ is $\mathcal{C}^s$.
\end{definition}

\begin{definition}[H{\"o}lder Function Class on $\cU$]\label{Def:Holder:cU}
Let $\cU$ be a compact and smooth manifold with an atlas $\{(U_{\alpha}, \phi_{\alpha})\}_{\alpha\in\mathcal{A}}$.
A function $f:~\mathcal{U}\to\mathbb{R}$ is said to belong to the H{\"o}lder function class $\HU$ with a positive integer $s$ and $\beta\in(0,1]$ if for any $\alpha\in\mathcal{A}$,
\begin{itemize}
\item
$f\circ\phi^{-1}_{\alpha}\in\mathcal{C}^{s}$ with $|D^{\cv s}(f\circ\phi^{-1}_{\alpha})|\le 1$, $\forall |\cv s|\le s$;
\item
The $s$-th order derivative of $f\circ\phi^{-1}_{\alpha}$ is H{\"o}lder continuous:
\[
\left|
D^{\cv s}(f\circ\phi^{-1}_{\alpha})\mid_{\phi_{\alpha}(\cv x_1)}
-
D^{\cv s}(f\circ\phi^{-1}_{\alpha})\mid_{\phi_{\alpha}(\cv x_2)}
\right|
\le 
\|\phi_{\alpha}(\cv x_1) - \phi_{\alpha}(\cv x_2)\|_2^{\beta},\quad \forall \cv x_1,\cv x_2\in U_{\alpha}, |\cv s|=s.
\]
\end{itemize}
\end{definition}

Throughout this paper, we assume that $\mathcal{U}$ is a smooth manifold, which means it has a $\mathcal{C}^\infty$ atlas.
For instance, the Euclidean space $\mathbb{R}^D$, the torus, and the unit sphere are smooth manifolds.
A smooth and compact manifold always induces a $\mathcal{C}^{\infty}$ partition of unity.
\begin{definition}[Partition of Unity~{\citep[Definition~13.4]{tu2010introduction}}]
A $\mathcal{C}^{\infty}$ partition of unity on a manifold $\cU$ is a collection of non-negative $\mathcal{C}^{\infty}$ functions $\rho_{\alpha}:~\cU\to\mathbb{R}_+$ with $\alpha\in\mathcal{A}$ such that
\begin{enumerate}
\item
the collection of supports, $\{\text{supp}(\rho_{\alpha})\}_{\alpha\in\cA}$, is locally finite, i.e., each point on $\cU$ has a neighborhood that meets only finitely many of $\text{supp}(\rho_{\alpha})$'s;
\item
$\sum_{\alpha\in\cA}\rho_{\alpha}=1$.
\end{enumerate}
\end{definition}
\begin{proposition}[Existence of a $\mathcal{C}^{\infty}$ partition of unity~{\citep[Theorem~13.7]{tu2010introduction}}]\label{Prop:exist:unity:partition}
Let $\{U_{\alpha}\}_{\alpha\in\cA}$ be an open cover of a compact smooth manifold $\cU$.
Then there exists a $\mathcal{C}^{\infty}$ partition of unity $\{\rho_{\alpha}\}_{\alpha\in\cA}$, where each $\rho_{\alpha}$ has a compact support so that $\text{supp}(\rho_{\alpha})\subseteq U_{\alpha}$ for some $\alpha\in\cA$.
\end{proposition}

We use reach to quantify the local curvature of a manifold $\mathcal{M}$~\citep{niyogi2008finding}.
Roughly speaking, a manifold with a large reach ``bends'' relatively slowly.
\begin{definition}[Reach]
The medial axis of $\mathcal{U}$ is defined as 
\[
\mathcal{T}(\mathcal{U})
=\left\{
\cv x\in\mathbb{R}^D:~
\exists \cv x_1\ne \cv x_2\in\mathcal{U}\text{ such that }\|\cv x-\cv x_1\|_2=\|\cv x-\cv x_2\|_2 = \inf_{\cv y\in\mathcal{U}}\|\cv x-\cv y\|_2
\right\}.
\]
The reach $\tau$ of $\mathcal{U}$ is the minimum distance between $\mathcal{U}$ and $\mathcal{T}(\mathcal{U})$:
\[
\tau = \inf_{\cv x\in\mathcal{T}(\mathcal{U}), \cv y\in\mathcal{U}}~\|\cv x-\cv y\|_2.
\]
\end{definition}

\subsection{Discrepancy between distributions}
{
We review some distance measures quantifying the difference between distributions.
First, we introduce the $\mathcal{L}_2$-divergence.
}
{
Consider an index set $\cA:=\{1,2,\ldots,|\cA|\}$.
For any discrete distribution $p$ on $\cA$, we associate $\cv p$ the corresponding \emph{probability mass vector} belonging to the simplex
 \[
 \Delta_{|\cA|}:=\left\{\cv x:~\sum_{i=1}^{|\cA|}\cv x[i]=1, \cv x[i]\ge0, \forall i\right\},
 \]
 where $\cv x[i]$, the $i$-th entry of $\cv x$, is the probability mass of $p$ at $i\in\cA$.
 \begin{definition}[$\mathcal{L}_2$-divergence]\label{Def:ell:2}
Let $p$ and $q$ be two discrete distributions on $\cA$.
The $\mathcal{L}_2$-divergence between distributions $p$ and $q$ is defined as 
\[
\mathcal{L}_2(p,q)=
\|\cv p-\cv q\|_2:=\left\{\sum_{i=1}^{|\cA|}\left(\cv p[i]- \cv q[i]\right)^2\right\}^{1/2},
\]
where $\cv p$ and $\cv q$ are the probability mass vectors corresponding to $p$ and $q$, respectively.
 \end{definition}
}
{
Next, we introduce the Wasserstein distance.
}
\begin{definition}[Wasserstein distance]\label{Def:Wasserstein}
The Wasserstein distance between two distributions $p$ and $q$ is defined as
\[
\W(p,q) = \inf_{\gamma\in\Gamma(p,q)}~\mathbb{E}_{\gamma}[\|\cv x-\cv y\|_2],
\]
where $\Gamma(p,q)$ denotes the set of joint distributions with marginal distributions being $p$ and $q$, respectively.
\end{definition}

{
Finally, we present the definition of integral probability metric.
\begin{definition}[Integral Probability Metric]\label{Def:int:IPM}
Given a function space $\cF$, the integral probability metric~(IPM) between two distributions $p$ and $q$ is defined as 
\[
d_{\cF}(p,q) = \sup_{f\in \cF}~\mathbb{E}_{p}[f(\cv x)] - \mathbb{E}_{q}[f(\cv x)].
\]
\end{definition}
Here we present some examples of IPM that will be used in the following of this work.
\begin{example}[H{\"o}lder IPM and NN IPM]\label{Example:special:IPM}
Let $\cU$ be a compact and smooth manifold.
When taking the function class in Definition~\ref{Def:int:IPM} as the H{\"o}lder function class on $\cU$, i.e., $\cF=\HU$, we obtain the H{\"o}lder IPM that measures the distance between two distributions $p$ and $q$ supported on $\cU$:
\[
d_{\HU}(p,q) = \sup_{f\in \HU}~\mathbb{E}_{p}[f(\cv x)] - \mathbb{E}_{q}[f(\cv x)].
\]
When taking $\cF$ as the neural network function class $\FNN(R,\kappa,L,t,K)$ defined in \eqref{Eq:specific:function:class}, we obtain the NN IPM that measures the distance between two distributions $p$ and $q$ supported on $\mathbb{R}^D$:
\[
d_{\FNN(R,\kappa,L,t,K)}(p,q) = \sup_{f\in \FNN(R,\kappa,L,t,K)}~\mathbb{E}_{p}[f(\cv x)] - \mathbb{E}_{q}[f(\cv x)].
\]
For notational simplicity, we write $d_{\FNN}(p,q)$ for $d_{\FNN(R,\kappa,L,t,K)}(p,q)$.
\end{example}
}

\section{Main results}\label{Sec:main:results}
This section contains our main theoretical analysis of two-sample tests for manifold data.
We begin with some assumptions on the manifold $\cU$.
\begin{assumption}\label{Assumption:cU}
The data generating distributions $p$ and $q$ are supported on $\cU$, where
\begin{enumerate}
\item[(I)]\label{Assumption:I:cU}
$\cU$ is a $d$-dimensional Riemannian compact and smooth manifold isometrically embedded in $\mathbb{R}^D$.
There exists a constant $B>0$ so that for any $\cv x\in\cU$, $|x_j|\le B$ for $j=1,\ldots,D$.
\item[(II)]
The reach of $\cU$ is $\tau>0$. 
\end{enumerate}
\end{assumption}
In the following, we propose a two-step test that applies to general distributions, provided that an atlas of $\cU$ is given, whose type-II risk scales in the order of $n^{-1/\max\{d,2\}}$.
When an atlas of $\cU$ is not given, we propose an H{\"o}lder IPM that applies for distributions with H{\"o}lder $(s,\beta)$-density functions, in which its type-II risk scales in the order of $n^{-(s+\beta)/d}$.
Then we propose the NN IPM test that approximates H{\"o}lder IPM test efficiently and achieves type-II risk in the order of $n^{-(s+\beta)/d}$, which is in the same order of the type-II risk as the H{\"o}lder IPM test.

\subsection{{Two-sample test with a given atlas}
}\label{Sec:2:A}
In this subsection, we focus on two-sample testing for general distributions where an atlas of $\cU$ is given.
We remark that as long as Assumption~\ref{Assumption:cU}(I) holds, there exists an atlas of $\cU$ containing finitely many charts.

\begin{assumption}\label{Assumption:Atlas}
A finite atlas of $\cU$ is given, denoted as $\{(U_{\alpha}, \phi_{\alpha})\}_{\alpha\in\cA}$ with $|\cA|<\infty$.
\end{assumption}

Based on the given atlas, we first argue that any distribution can be decomposed as a convex combination of distributions supported on each local neighborhood of $\cU$.
By Proposition~\ref{Prop:exist:unity:partition}, there exists a $\mathcal{C}^{\infty}$ partition of unity $\{\rho_{\alpha}\}_{\alpha\in\cA}$ such that $\text{supp}(\rho_{\alpha})\subseteq U_{\alpha}$ for each $\alpha\in\cA$.
Denote the event $\mathcal{E}_{\alpha} = \{\cv x:~\cv x\in\text{supp}(\rho_{\alpha})\}$.
For any distribution $p$ supported on $\cU$, denote the probability of the event $\mathcal{E}_{\alpha}$ as 
\[
p^{(\alpha)}:=
\mbE_{\cv x\sim p}[ 
\dsI\{\cv x\in \mathcal{E}_{\alpha}\}
].
\]
Denote the distribution of $p$ conditioned on the event $\mathcal{E}_{\alpha}$ as $p\mid\mathcal{E}_{\alpha}$.
For any event, $A$ it holds that
\[{
\big(p\mid\mathcal{E}_{\alpha}\big)(A):= 
\frac{
\mbE_{\cv x\sim p}[\dsI\{\cv x\in A\cap \mathcal{E}_{\alpha}\}]
}{p^{(\alpha)}}.}
\]
As a consequence, any probability distribution $p$ supported on $\cU$ can be decomposed as
\begin{equation}
p = \sum_{\alpha\in\cA}p^{(\alpha)}\times \big(p\mid\mathcal{E}_{\alpha}\big),\label{Eq:decomposition:dis:cU}
\end{equation}
where operators $\sum$ and $\times$ represent summation and scalar multiplication for real-valued functions.
The decomposition \eqref{Eq:decomposition:dis:cU} indicates that any distribution supported on the manifold can be viewed as a finite mixture distribution.

Therefore, for any two distributions $p$ and $q$, the null hypothesis $H_0:~p=q$ holds if and only if the following two null hypotheses both hold:
\[
H_0^{\tI}:~p^{(\alpha)}=q^{(\alpha)},\forall\alpha\in\cA,\quad
H_0^{\tII}:~p\mid\mathcal{E}_{\alpha}=q\mid\mathcal{E}_{\alpha}, \forall \alpha\in\cA.
\]
We propose a two-sample test $\mathcal{T}_{\text{Atlas}}$ that decides whether to reject the null hypothesis $H_0:~p=q$, which is performed in two steps.
\begin{enumerate}
\item[\tI]
Reject the null hypothesis $H_0^{\tI}$ if there exists $\alpha\in\cA$ so that $p^{(\alpha)}\ne q^{(\alpha)}$.
\item[\tII]
If the null hypothesis $H_0^{\tI}$ is not rejected, then reject the null hypothesis $H_0^{\tII}$ if there exists $\alpha\in\cA$ so that $p\mid\mathcal{E}_{\alpha}\ne q\mid\mathcal{E}_{\alpha}$.
\end{enumerate}
An overview of the proposed two-sample test $\mathcal{T}_{\text{Atlas}}$ is provided in Figure~\ref{Fig:diagram:agn}.
In the following, we first explain the detailed procedures of Step~\tI~and \tII, and then present the testing properties of $\mathcal{T}_{\text{Atlas}}$.
\paragraph{Step~\tI: Hypothesis Testing on $H_0^{\tI}$.}
Define the probability mass vectors $\cv p:=\{p^{(\alpha)}\}_{\alpha\in\cA}$ and $\cv q:=\{q^{(\alpha)}\}_{\alpha\in\cA}$.
In order to examine the hypothesis $H_0^{\tI}$ that whether the probability mass vectors $\cv p$ and $\cv q$ are the same or not, we first formulate sample estimates of them.
Based on collected observations $\cv x^n$ and $\cv y^n$, we formulate probability mass vectors 
$\hat{\cv p}=\{\hat{p}^{(\alpha)}\}_{\alpha\in\cA}$ and $\hat{\cv q}=\{\hat{q}^{(\alpha)}\}_{\alpha\in\cA}$, respectively, where
\begin{equation}
\hat{p}^{(\alpha)}=\frac{1}{n}\sum_{i=1}^n\dsI\{\cv x_i\in \mathcal{E}_{\alpha}\},\quad
\hat{q}^{(\alpha)}=\frac{1}{n}\sum_{i=1}^n\dsI\{\cv y_i\in \mathcal{E}_{\alpha}\}.
\label{Eq:p:q:alpha}
\end{equation}
We design a two-sample test $\mathcal{T}_{\text{Atlas}}^{\tI}$ using the $\mathcal{L}_2$-divergence (see Definition~\ref{Def:ell:2}) as the testing statistic, in which the null hypothesis $H_0^{\tI}$ is rejected if 
\begin{equation}\label{Eq:two:sample:test:multi:nominal}
\|\hat{\cv p}-\hat{\cv q}\|_2\ge \ell_n^{\tI},
\end{equation}
where the threshold $\ell_n^{\tI}$ is defined as
\begin{equation}
\ell_{n}^{\tI}:=\sqrt{2/n}\left(1 + \sqrt{2\log\eta^{-1}}\right).\label{Eq:threshold:tI}
\end{equation}

 \begin{figure}[t!]
\centering
\includegraphics[width=\textwidth]{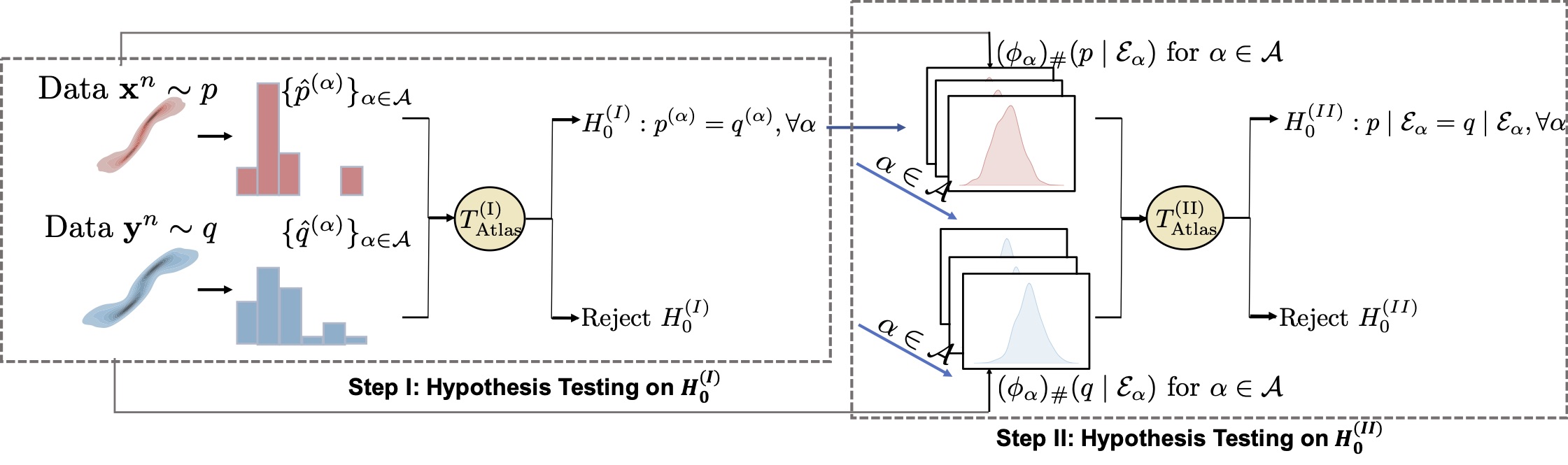}
\caption{
An overview of the proposed two-sample test with agnostic data, which consists of two steps: (I) perform a test that examines whether $p^{(\alpha)}=q^{(\alpha)}, \forall\alpha\in\cA$, i.e., the probabilities of supporting on $\text{supp}(\rho_{\alpha})$ are the same or not.
(II) When the null hypothesis $H_0^{\tI}$ in Step~(I) is not rejected, perform a test to decide whether $p\mid\mathcal{E}_{\alpha}=q\mid\mathcal{E}_{\alpha}, \forall\alpha\in\cA$, i.e., whether these two distributions conditioned on the event of supporting on $\text{supp}(\rho_{\alpha})$ are the same or not.
This step is by transferring these distributions based on chart mappings $\{\phi_{\alpha}\}_{\alpha\in\cA}$.
}
\label{Fig:diagram:agn}
\end{figure}
\paragraph{Step~\tII: Hypothesis Testing on $H_0^{\tII}$.}
For any two distributions $p,q$ supported on $\cU$, in order to examine the null hypothesis $H_0^{\tII}$, we consider the following testing function:
\begin{equation}
T(p,q):=\max_{\alpha\in\cA}~\cW\big(
(\phi_{\alpha})_{\#}(p\mid\mathcal{E}_{\alpha}),
(\phi_{\alpha})_{\#}(q\mid\mathcal{E}_{\alpha})
\big),\label{Eq:cv:p:T:p:q}
\end{equation}
where $\mathcal{W}(\cdot,\cdot)$ is the Wasserstein distance (see Definition~\ref{Def:Wasserstein}).
The testing function $T(p,q)=0$ if and only if $(\phi_{\alpha})_{\#}(p\mid\mathcal{E}_{\alpha})=(\phi_{\alpha})_{\#}(q\mid\mathcal{E}_{\alpha}), \forall\alpha$, i.e, if and only if the null hypothesis $H_0^{\tII}: p\mid\mathcal{E}_{\alpha}=q\mid\mathcal{E}_{\alpha}, \forall\alpha$ holds.
This testing function characterizes the discrepancy between projected distributions, which can be viewed as a special projected Wasserstein distance~\citep{wang2020twosample, wang2020twosamplekernel} with $\phi_{\alpha}, \alpha\in\cA$ being the candidate projection mappings.

Denote by $\hat{p}_n$ and $\hat{q}_n$ the empirical distributions from observations $\cv x^n$ and $\cv y^n$, respectively:
\begin{equation}
\hat{p}_n=\frac{1}{n}\sum_{i=1}^n\delta_{\cv x_i},\quad
\hat{q}_n=\frac{1}{n}\sum_{i=1}^n\delta_{\cv y_i}.\label{Eq:empirical:hp:hq}
\end{equation}
We design a two-sample test $\mathcal{T}_{\text{Atlas}}^{\tII}$ so that the null hypothesis $H_0^{\tII}$ is rejected if the empirical counter-part of the testing function $T(p,q)$ is above a pre-specified threshold:
\begin{equation}
\label{Stat:project:wasserstein}
T(\hat{p}_n, \hat{q}_n) \ge \ell_n^{\tII},
\end{equation}
where the threshold is defined as
\begin{equation}
\ell_{n}^{\tII}:=
\left\{
\begin{aligned}
&\left(
\frac{\log(c_1\eta^{-1})}{c_2n}
\right),&\quad\text{if $n<\frac{\log(c_1\eta^{-1})}{c_2}$},\\
&\left(
\frac{\log(c_1\eta^{-1})}{c_2n}
\right)^{\max\{d,3\}},&\quad\text{if $n\ge\frac{\log(c_1\eta^{-1})}{c_2}$}
\end{aligned}
\right.,\label{Eq:threshold:tII}
\end{equation}
with $c_1,c_2$ being some constants depending on $d,\cU$.

\begin{center}
\ifnum\paperversion=2
\begin{algorithm}[t!]
\caption{
Algorithm for Two-sample testing with a Given Atlas
} 
\fi
\ifnum\paperversion=1
\begin{algorithm}[Algorithm for Two-sample testing with a Given Atlas]
\fi
\label{Alg:agn}
\begin{algorithmic}[1]\label{Alg:ang:test}
\ifnum\paperversion=1
\textbf{\quad}
\fi
\REQUIRE{
Training samples $\cv x^n$ and $\cv y^n$, and a given atlas of $\cU$} 
\STATE{Formulate sample estimates of $\cv p, \cv q$, denoted as $\hat{\cv p}, \hat{\cv q}$, according to \eqref{Eq:p:q:alpha}.}\hfill\COMMENT{Line~1-5: Perform Step~\tI}
\STATE{Compute testing statistic $\|\hat{\cv p} - \hat{\cv q}\|_2$ according to \eqref{Eq:two:sample:test:multi:nominal}.}
\IF{$\|\hat{\cv p} - \hat{\cv q}\|_2\ge\ell_n^{\tI}$} 
\STATE{Reject $H_0$.}
\ELSE
\STATE{Formulate empirical distributions $\hat{p}_n,\hat{q}_n$ according to \eqref{Eq:empirical:hp:hq}.}\hfill\COMMENT{Line~6-12: Perform Step~\tII}
\STATE{
Compute testing statistic $T(\hat{p}_n, \hat{q}_n)$ according to \eqref{Eq:cv:p:T:p:q}.
}
\IF{$T(\hat{p}_n, \hat{q}_n)\ge\ell_n^{\tII}$} 
\STATE{Reject $H_0$.}
\ELSE
\STATE{Not reject $H_0$.}
\ENDIF
\ENDIF
\\
\textbf{Return} the decision of whether rejecting $H_0$ or not.
\end{algorithmic}
\end{algorithm}
\end{center}
With those notations in force, we summarize the overall procedure of $\mathcal{T}_{\text{Atlas}}$ in Algorithm~\ref{Alg:ang:test}.
In the following, we present an analysis of the testing properties of $\mathcal{T}_{\text{Atlas}}$.
\begin{theorem}[Test Properties of $\mathcal{T}_{\text{Atlas}}$]\label{Theorem:testing:agn}
Fix a level $\eta\in(0,1/2)$.
Under Assumption~\ref{Assumption:cU}(I) and \ref{Assumption:Atlas}, suppose 
\[
\min_{\alpha\in\cA}~\left(p^{(\alpha)}\land q^{(\alpha)}\right)\ge c,
\]
where $c$ is a positive constant, then it holds that:
\begin{itemize}
\item
\emph{Risk}: The two-sample test $\mathcal{T}_{\text{Atlas}}$ has type-I risk at most $\eta$;
\item
\emph{Power}: Under $H_1$, additionally assume the sample size $n$ is sufficiently large so that 
\begin{align}
\Delta_n^{\tI}&:=\|{\cv p} - {\cv q}\|_2 - \ell_{n}^{\tI} - 4/\sqrt{n}>0,\quad
\text{ or }\quad
\Delta_n^{\tII}:=T(p,q) - \ell_{n}^{\tII}>0,
\label{Eq:Delta:tI:tII}
\end{align}
then the power of $\mathcal{T}_{\text{Atlas}}$ is $1-c'n^{-1/\max\{d,2\}}\log n$, where the constant $c'$ depends on $c,d,\cU, \Delta_n^{\tI}, \Delta_n^{\tII}$.
\end{itemize}
\end{theorem}
The proof of Theorem~\ref{Theorem:testing:agn} is provided in Section~\ref{Sec:proof:Thm:agn:testing}.
The following remark explains how to obtain tighter thresholds in practice so that the testing power can be improved.
\begin{remark}[Bootstrap]\label{Remark:bootstrap}
The thresholds developed in $\mathcal{T}_{\text{Atlas}}^{\tI}$ and $\mathcal{T}_{\text{Atlas}}^{\tII}$ are usually conservative as they are based on non-asymptotic concentration inequalities.
In practice, we use the bootstrap approach to obtain tighter thresholds and then perform two-sample tests.
In the following, we discuss how to perform bootstrap for $\mathcal{T}_{\text{Atlas}}^{\tII}$ as an example.
Denote by $\cv g^{2n}=(\cv x_1,\ldots,\cv x_n,\cv y_1,\ldots,\cv y_n)$ the pooled data, and $\hat{\gamma}_{2n}=\frac{1}{2n}\sum_{i=1}^{2n}\delta_{\cv g_i}$.
Then for the $b$-th batch, let $\cv x_1^{(b)}, \ldots, \cv x_n^{(b)}$ and $\cv y_1^{(b)},\ldots,\cv y_n^{(b)}$ be i.i.d. samples from $\hat{\gamma}$, and denote by $\hat{p}_n^{(b)}$ and $\hat{q}_n^{(b)}$ the corresponding empirical measures.
Then define the bootstrap statistic
\[
T_{n}^{(b)} = T(\hat{p}_n^{(b)}, \hat{q}_n^{(b)}).
\]
The threshold to make the test (approximately) at level $\eta$ becomes
\[
T_{n, 1-\eta} = \inf\left\{t:~\text{Pr}\left(
T_{n}^{(b)}\le t
\right)\ge 1-\eta\right\}.
\]
Applying the argument in \cite{van1996weak}
implies that a two-sample test that rejects $H_0$ when 
\[
T(\hat{p}_n, \hat{q}_n) \ge T_{n, 1-\eta}
\]
is asymptotically consistent with level $\eta$.
Moreover, the value of $T_{n, 1-\eta}$ can be estimated by computing $T_{n}^{(b)}, b=1,2,\ldots,N_B$ for a sufficiently large $N_B$ and then taking the empirical $(1-\eta)$-quantile.
\end{remark}

\subsection{Two-sample test when an atlas is not given}
\label{Sec:test:structured}
Now we consider the case where an atlas of $\cU$ is not given, but data-generating distributions $p$ and $q$ supported on $\cU$ have smooth density functions.
In this case, we use the H{\"o}lder IPM $d_{\HU}(\cdot,\cdot)$ in Example~\ref{Example:special:IPM} as the testing statistic.
The framework of the proposed IPM-based two-sample test is provided in Figure~\ref{Fig:structured}.
We find that the H{\"o}lder IPM can identify any two distinct distributions supported on $\cU$ that have smooth density functions.
The formal result is summarized in Proposition~\ref{Proposition:discriminative}, the proof of which is provided in Section~\ref{Thm:test:property:proof}.

\begin{figure}[t!]
\centering
\includegraphics[width=\textwidth]{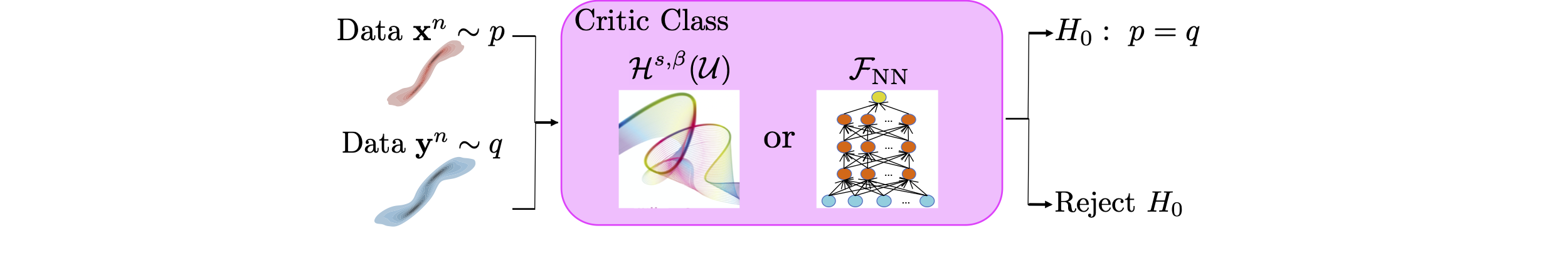}
\caption{The architecture of proposed IPM-based test with Structured data.
We take the critic of IPM as a function in H{\"o}lder space $\cH^{s,\beta}(\cU)$ (see Section~\ref{Sec:test:structured}) or a neural network function (see Section~\ref{Sec:estimate:IPM:NN}).
}
\label{Fig:structured}
\end{figure}

\begin{assumption}[Smooth Densities on $\cU$]\label{Assumption:smooth:density}
The data generating distributions $p$ and $q$ are supported on $\cU$ with the well-defined density functions $\mathfrak{h}_p, \mathfrak{h}_q\in \HU$, respectively.
\end{assumption}
\begin{proposition}[Discriminative Property]
\label{Proposition:discriminative}
Under Assumption \ref{Assumption:cU}(I) and Assumption~\ref{Assumption:smooth:density},
the H{\"o}lder IPM satisfies the discriminative property, i.e., $d_{\HU}(p,q)=0$ if and only if $p=q$.
\end{proposition}

Proposition~\ref{Proposition:discriminative} reveals that as long as densities of target distributions are H{\"o}lder $(s,\beta)$-smooth, the H{\"o}lder IPM can identify the difference between them.
We thereby design a two-sample test $\mathcal{T}_{\text{H{\"o}lder}}$ so that the null hypothesis $H_0$ is rejected if
\begin{equation}\label{Eq:test:statistic:case:II}
d_{\HU}(\hat{p}_n,\hat{q}_n)\ge \ell_{n}.
\end{equation}
where $\ell_{n}$ is a threshold, and $\hat{p}_n,\hat{q}_n$ are empirical distributions from $\cv x^n, \cv y^n$, respectively. %
The following result summarizes the testing properties of $\mathcal{T}_{\text{H{\"o}lder}}$.
\begin{theorem}[Test Properties of $\mathcal{T}_{\text{H{\"o}lder}}$]
\label{Theorem:test:property}
Fix a level $\eta\in(0,1/2)$.
Under Assumption~\ref{Assumption:cU}(I) and Assumption~\ref{Assumption:smooth:density},
suppose we set the threshold of the two-sample test $\cT_{\text{H{\"o}lder}}$ defined in \eqref{Eq:test:statistic:case:II} as
\[
\ell_{n}=c_1n^{-\frac{s+\beta}{d}}+\sqrt{4\log\frac{2}{\eta}}n^{-\frac{1}{2}},
\]
for some constant $c_1$ depending on $s,\beta,d,|\mathcal{A}|$, 
then it holds that:
\begin{itemize}
\item
{\it Risk}: 
The type-I risk of $\cT_{\text{H{\"o}lder}}$ is at most $\eta$;
\item
{\it Power}:
Under $H_1$, suppose the sample size $n$ is sufficiently large so that 
\[
\Delta_n:=d_{\HU}(p,q)-\ell_{n}>0,
\]
the power of $\mathcal{T}_{\text{H{\"o}lder}}$ is at least $1 - \frac{c_2}{\Delta_n}\cdot n^{-\frac{s+\beta}{d}}$, where $c_2$ is a constant depending on $s,\beta,d,|\mathcal{A}|$.
\end{itemize}
\end{theorem}

The proof of Theorem~\ref{Theorem:test:property} is provided in Section~\ref{Thm:test:property:proof}.
This theorem indicates that as long as the sample size $n$ grows sufficiently large so that $\Delta_n=O(1)$, the type-II risk of $\mathcal{T}_{\text{H{\"o}lder}}$ with a fixed level $\eta$ scales of $O\left(n^{-\frac{s+\beta}{d}}\right)$.

\subsection{Estimating H{\"o}lder IPM by neural networks}
\label{Sec:estimate:IPM:NN}
To mitigate the heavy computation burden of directly evaluating the H{\"o}lder IPM, we parameterize the H{\"o}lder function class using neural networks.
In other words, the H{\"o}lder IPM testing statistic in \eqref{Eq:test:statistic:case:II} is approximated using the IPM with a neural network function class
\[
d_{\FNN}(\hat{p}_n,\hat{q}_n),
\]
where the class $\FNN(R,\kappa,L,t,K)$ is defined in \eqref{Eq:specific:function:class}. 
We design a two-sample test $\mathcal{T}_{\text{NN}}$ so that the null hypothesis $H_0$ is rejected if
\begin{equation}\label{Eq:test:statistic:case:II:app}
d_{\FNN}(\hat{p}_n,\hat{q}_n)\ge \ell_{n},
\end{equation}
where $\ell_{n}$ is a threshold.
The following theorem summarizes the testing properties of $\mathcal{T}_{\text{NN}}$.
\begin{theorem}[Test Properties of $\mathcal{T}_{\text{NN}}$]
\label{Theorem:test:property:app}
Fix a level $\eta\in(0,1/2)$.
Under Assumption~\ref{Assumption:cU} and Assumption~\ref{Assumption:smooth:density}, we specify hyper-parameters of $\FNN(R,\kappa,L,t,K)$ as 
\[
\begin{aligned}
R&=1,\quad \kappa=O(\max\{1, B, \sqrt{d}, \tau^2\}),\quad L=O\left(
\frac{s+\beta}{2(s+\beta)+d}\log(nD)
\right),\quad t={O}\left(
n^{\frac{d}{2(s+\beta)+d}}+D
\right),\quad\text{and }\\
K&={O}\left(
\frac{s+\beta}{2(s+\beta)+d}n^{\frac{d}{2(s+\beta)+d}}\log n
+\frac{s+\beta}{2(s+\beta)+d}D\log n + D\log D
\right),
\end{aligned}
\]
and set the threshold of the two-sample test $\cT_{\text{NN}}$ defined in \eqref{Eq:test:statistic:case:II:app} as
\[
\ell_{n}=c_1n^{-\frac{s+\beta}{2(s+\beta)+d}}\log^2n+\sqrt{4\log\frac{2}{\eta}}n^{-\frac{1}{2}},
\]
where $c_1$ is some constant depending on $\tau,\sqrt{d},B$ and linearly depending on $\sqrt{D}$.
Then it holds that:
\begin{itemize}
\item
{\it Risk}: 
The type-I risk of $\cT_{\text{NN}}$ is at most $\eta$;
\item
{\it Power}:
Under $H_1$, suppose the sample size $n$ is sufficiently large so that 
\[
\Delta_n:=d_{\HU}(p,q)- \left[\ell_{n} + c_2n^{-\frac{s+\beta}{2(s+\beta)+d}}\right]>0,
\]
the power of $\mathcal{T}_{\text{NN}}$ is at least $1 - \frac{c_3}{\Delta_n}\cdot n^{-\frac{s+\beta}{d}}$, where $c_2$ is a constant depending on $\tau,\sqrt{d},B$ and linearly depending on $\sqrt{D}$, and $c_3$ is a constant depending on $s,\beta,d,|\mathcal{A}|$.
\end{itemize}
\end{theorem}

The proof of Theorem~\ref{Theorem:test:property:app} is provided in Section~\ref{Sec:proof:Theorem:test:property:app}.
This theorem reveals that the type-II risk of $\mathcal{T}_{\text{NN}}$ with a fixed level $\eta$ scales in the order of $n^{-(s+\beta)/d}$, which is at the same rate of the type-II risk of $\mathcal{T}_{\text{H{\"o}lder}}$.
Moreover, we note that the thresholds in Theorem~\ref{Theorem:test:property} and Theorem~\ref{Theorem:test:property:app} are conservative bounds.
To design practical tests, we compute those thresholds using the bootstrap idea similar to Remark~\ref{Remark:bootstrap}.

\section{Proof of Theorem~\ref{Theorem:testing:agn}}
\label{Sec:proof:Thm:agn:testing}

We first investigate the testing properties of $\mathcal{T}_{\text{Atlas}}^{\tI}$ and $\mathcal{T}_{\text{Atlas}}^{\tII}$, which are summarized in Lemma~\ref{Lemma:testing:Ang:I}, \ref{Lemma:N:lower:bound}, and \ref{Lemma:test:property:projected:samples}.
Based on those technical lemmas, we give the proof of 
Theorem~\ref{Theorem:testing:agn}.

\subsection{Testing performance of $\mathcal{T}_{\text{Atlas}}^{\tI}$}
In the following lemma, we characterize the testing performance of $\mathcal{T}_{\text{Atlas}}^{\tI}$.
\begin{lemma}[Test Properties of $\mathcal{T}_{\text{Atlas}}^{\tI}$]
\label{Lemma:testing:Ang:I}
Fix a level $\eta\in(0,1/2)$.
Under Assumption~\ref{Assumption:cU}(I) and \ref{Assumption:Atlas},
it holds that:
\begin{itemize}
\item
{\it Risk}: 
The type-I risk of $\cT_{\text{Atlas}}^{\tI}$ is at most $\eta$;
\item
{\it Power}:
Under $H_1$, suppose the sample size $n$ is sufficiently large so that $\Delta_n^{\tI}>0$ (see \eqref{Eq:Delta:tI:tII}),
the power of $\mathcal{T}_{\text{Atlas}}^{\tI}$ is at least $1 - 2\exp\left(
-\frac{(\Delta_n^{\tI})^2n}{4}
\right)$.
\end{itemize}
\end{lemma}
This lemma shows that when the sample size $n$ grows sufficiently large so that $\Delta_n^{\tI}=\Theta(1)$, the type-II risk of $\mathcal{T}_{\text{Atlas}}^{\tI}$ with a fixed type-I risk $\eta$ decays to zero exponentially fast in the sample size $n$.

As demonstrated in Proposition~3 of \cite{xie2021sequential}, the application of $\mathcal{L}_2$-divergence is promising since it achieves (up to a logarithmic factor) the optimal order of sample complexity for building reliable two-sample tests.
It is also possible to consider other tests to improve the constant hidden in the sample complexity rate, where $\mathcal{L}_2$-divergence serves as an example among those near-optimal choices.

\begin{proof}[Proof of Lemma~\ref{Lemma:testing:Ang:I}]
The statistic $\|\cv p-\cv q\|_2$ can be written as the MMD with the kernel function 
\[
k(x,y) = \inp{x}{y},\quad x,y\in\mathbb{R}^{|\cA|},
\]
and $0\le k(x,y)\le 1$.
Moreover, the statistic $\|\hat{\cv p}-\hat{\cv q}\|_2$ is nothing but its biased MMD sample estimate.
Recall that \cite{Gretton12} established concentration inequalities for MMD sample estimate under the null hypothesis and more general cases, which are summarized in the following two lemmas.
\begin{lemma}[Measure Concentration of MMD Under Null Hypothesis~{\citep[Corollary~9]{Gretton12}}]\label{Lemma:null:MMD}
Let $p,q$ be two probability distributions, and $\hat{p}_n,\hat{q}_n$ be the empirical measures from $n$ samples generated from $p,q$, respectively.
Suppose the MMD function is constructed based on the kernel function $k(x,y)$ with $0\le k(x,y)\le K$.
When $p=q$, it holds that
\[
\Pr\left\{
\text{MMD}(\hat{p}_n,\hat{q}_n)
>
\sqrt{2K/n}\left(
1+\sqrt{2\log\alpha^{-1}}
\right)
\right\}\le 1-\alpha.
\]
\end{lemma}
\begin{lemma}[Measure Concentration of MMD~{\citep[Theorem~7]{Gretton12}}]\label{Lemma:MMD}
Let $p,q$ be two probability distributions, and $\hat{p}_n,\hat{q}_n$ be the empirical measures from $n$ samples generated from $p,q$, respectively.
Suppose the MMD function is constructed based on the kernel function $k(x,y)$ with $0\le k(x,y)\le K$.
Then it holds that
\[
\Pr\left\{
\Big|\text{MMD}(p,q) - \text{MMD}(\hat{p}_n,\hat{q}_n)\Big|
>
4
(K/n)^{1/2}
+\epsilon
\right\}\le 2\exp\left(
\frac{-\epsilon^2n}{4K}
\right).
\]
\end{lemma}
Based on the concentration inequality in Lemma~\ref{Lemma:null:MMD}, under the null hypothesis $H_0^{\tI}$, the relation
\[
\|\hat{\cv p}-\hat{\cv q}\|_2\le \ell_n^{\tI}:=\sqrt{2/n}\left(
1 + \sqrt{2\log\eta^{-1}}
\right)
\]
holds with probability at least $1-\eta$.
Under the alternative hypothesis $H_1^{\tI}$, based on the concentration inequality in Lemma~\ref{Lemma:MMD}, for any $\epsilon>0$, the relation
\[
\Big|
\|{\cv p}-{\cv q}\|_2
-
\|\hat{\cv p}-\hat{\cv q}\|_2
\Big|\le 4/\sqrt{n} + \epsilon
\]
holds with probability at least $1-2\exp\left(-\frac{\epsilon^2n}{4}\right)$.
Taking $\epsilon = \Delta_n^{\tI}>0$ implies that
\[
\Pr\left\{
\Big|
\|{\cv p}-{\cv q}\|_2
-
\|\hat{\cv p}-\hat{\cv q}\|_2
\Big|\le \|{\cv p} - {\cv q}\|_2 - \ell_n^{\tI}
\right\}\ge  1-2\exp\left(-\frac{\epsilon^2n}{4}\right).
\]
Consequently, it holds that
\[
\Pr\left\{
\|\hat{\cv p}-\hat{\cv q}\|_2
\ge \ell_n^{\tI}
\right\}\ge  1-2\exp\left(-\frac{\epsilon^2n}{4}\right).
\]
The proof of Lemma~\ref{Lemma:testing:Ang:I} is completed.
\end{proof}

\subsection{Testing performance of $\mathcal{T}_{\text{Atlas}}^{\tII}$}
Next, we study the testing properties of $\mathcal{T}_{\text{Atlas}}^{\tII}$.
For notational simplicity, we define $N$ as the minimum number of samples supported on local neighborhoods of $\cU$:
\begin{align}
N&:=\min_{\alpha\in\cA}
\bigg\{
\#\{\cv x_i\in \text{supp}(\rho_{\alpha})\}\land
\#\{\cv y_i\in \text{supp}(\rho_{\alpha})\}
\bigg\}.\label{Eq:N:expression}
\end{align}
The following lemma suggests that $N=\Theta(n)$ holds with high probability.
\begin{lemma}[Lower Bound on $N$]\label{Lemma:N:lower:bound}
With probability at least $1-e^{-\sqrt{n}}$, it holds that
\[
N\ge n\cdot\min_{\alpha\in\cA}~\left(p^{(\alpha)}\land q^{(\alpha)}\right) - \sqrt{n\cdot 4|\cA|\ln 2 + n^{3/2}\cdot 2|\cA|}.
\]
\end{lemma}

\begin{proof}[Proof of Lemma~\ref{Lemma:N:lower:bound}]
Based on the concentration inequality from Theorem~2.1 in \cite{weissman2003inequalities}, with probability at least $1-\delta$, it holds that
\[
\min_{\alpha\in\cA}~\hat{p}^{(\alpha)}\ge \min_{\alpha\in\cA}~p^{(\alpha)} - \sqrt{\frac{2|\cA|\ln(2/\delta)}{n}}.
\]
The same result holds for the concentration of $\min_{\alpha\in\cA}~\hat{q}^{(\alpha)}$ in terms of $\min_{\alpha\in\cA}~q^{(\alpha)}$, then with probability at least $1-2\delta$, it holds that
\[
N\ge n*\min_{\alpha\in\cA}~\left(p^{(\alpha)}\land q^{(\alpha)}\right) - \sqrt{2|\cA|n\ln(2/\delta)
}.
\]
We take $\delta=\frac{1}{2}e^{-\sqrt{n}}$ to obtain that
\[
\Pr\left\{
N\ge n*\min_{\alpha\in\cA}~\left(p^{(\alpha)}\land q^{(\alpha)}\right) - \sqrt{2|\cA|\big(2n\ln 2 + n^{3/2}\big)}
\right\}\ge 1 - e^{-\sqrt{n}}.
\]
\end{proof}

In the following lemma, we characterize the testing performance of $\mathcal{T}_{\text{Atlas}}^{\tII}$.
In particular, we define the threshold for two-sample testing as 
\[
\ell_{n}^{\tII}:=
\left\{
\begin{aligned}
&\left(
\frac{\log(c_1\eta^{-1})}{c_2N}
\right),&\quad\text{if $N<\frac{\log(c_1\eta^{-1})}{c_2}$},\\
&\left(
\frac{\log(c_1\eta^{-1})}{c_2N}
\right)^{\max\{d,3\}},&\quad\text{if $N\ge\frac{\log(c_1\eta^{-1})}{c_2}$}.
\end{aligned}
\right.
\]
Since $N=\Theta(n)$ holds with high probability, in Section~\ref{Sec:2:A} we write the threshold $\ell_{n}^{\tII}$ in terms of $n$ for the sake of notational simplicity.
\begin{lemma}[Test Properties of $\mathcal{T}_{\text{Atlas}}^{\tII}$]
\label{Lemma:test:property:projected:samples}
Fix a level $\eta\in(0,1/2)$.
Under Assumption~\ref{Assumption:cU}(I) and \ref{Assumption:Atlas}, 
it holds that
\begin{itemize}
\item
{\it Risk}: 
The type-I risk of $\cT_{\text{Atlas}}^{\tII}$ is at most $\eta$;
\item
{\it Power}:
Under $H_1$, suppose the sample size $n$ is sufficiently large so that $\Delta_n^{\tII}>0$ (see \eqref{Eq:Delta:tI:tII}), 
then the power of $\mathcal{T}_{\text{Atlas}}^{\tII}$ is at least $1 - \frac{c_3\log(N)}{\Delta_n^{\tII}}\cdot N^{-1/\max\{d,2\}}$, where $c_3$ is a constant depending on $d,\cU$.
\end{itemize}
\end{lemma}
Lemma~\ref{Lemma:test:property:projected:samples} reveals that, as long as $N=\Theta(n)$ and the sample size $n$ grows sufficiently large so that $\Delta_n^{\tII}=O(1)$, the type-II risk of $\mathcal{T}_{\text{Atlas}}^{\tII}$ with a fixed level $\eta$ scales of $\tilde{O}(n^{-1/\max\{d,2\}})$, where the notation $\tilde{O}(\cdot)$ hides constants related to $\log n, d, \cU$.

The Wasserstein distance is a suitable measure for two-sample testing since it takes into account the underlying geometry of the sample space, and it is well-defined for all probability measures, no matter their absolute continuity.
Unfortunately, the sample complexity of the empirical Wasserstein distance suffers from the curse of dimensionality, which means the performance of standard Wasserstein tests degrades quickly for increasing dimensions~\citep{nilesweed2019estimation, weed2017sharp, chen2020asymptotics}.
In contrast, our Wasserstein test $\mathcal{T}_{\text{Atlas}}^{\tII}$ defined in \eqref{Stat:project:wasserstein} alleviates such an issue because the projected samples are in low dimensions.

\begin{proof}[Proof of Lemma~\ref{Lemma:test:property:projected:samples}]
Let us begin with some preliminary results on statistical properties of empirical Wasserstein distance provided that the data generating distribution is supported on a bounded space.
\begin{lemma}[Measure Concentration of Wasserstein Distsance~{\citep[Theorem~2]{fournier2015rate}}]\label{Lemma:concentration:wasserstein}
Let $p$ be a probability distribution in $\mathbb{R}^d$ supported on a bounded space of radius $b$, and $\hat{p}_n$ be the empirical measure from $n$ samples generated from $p$.
It holds that
\[
\Pr\left\{
\mathcal{W}(p,\hat{p}_n)\ge \epsilon
\right\}\le c_1\exp(-c_2n\epsilon^2)\cdot 1\{\epsilon>1\}+ c_1\exp(-c_2n\epsilon^{\max\{d,3\}})\cdot 1\{\epsilon\le 1\},
\]
where 
$c_1,c_2$ are constants that depend on $b$ and $d$ only.
\end{lemma}

\begin{lemma}[Expectation of Empirical Wasserstein Distsance~{\citep[Theorem~1]{fournier2015rate}}]\label{Lemma:expectation:wasserstein}
Let $p$ be a probability distribution in $\mathbb{R}^d$ supported on a bounded space of radius $b$, and $\hat{p}_n$ be the empirical measure from $n$ samples generated from $p$.
It holds that
\[
\mathbb{E}[\mathcal{W}(p,\hat{p}_n)]
\le 
\left\{
\begin{aligned}
c_1n^{-1/2}\log(n),&\quad \text{if $d=2$},\\
c_1n^{-1/\max\{d,2\}},&\quad\text{if $d\ne 2$}
\end{aligned}
\right.
\]
and $c_1$ is a constant that depend on $b$ and $d$ only.
\end{lemma}

Note that for any $\alpha\in\cA$ and $\cv x\in\text{supp}(\rho_{\alpha})$, the output $\phi_{\alpha}(\cv x)$ is also bounded, since the image of the continuous function $\phi_{\alpha}$ valued on a compact manifold is also compact.
Therefore, concentration results in Lemma~\ref{Lemma:concentration:wasserstein} and Lemma~\ref{Lemma:expectation:wasserstein} can be applied for Wasserstein distance between $(\psi_{\alpha})_\#(\hat{p}_n\mid\mathcal{E}_{\alpha})$ and $(\psi_{\alpha})_\#(\hat{q}_n\mid\mathcal{E}_{\alpha})$.
Under the null hypothesis $H_0^{\tII}$, applying the concentration result in Lemma~\ref{Lemma:concentration:wasserstein} implies that, for fixed $\alpha\in\cA$, with probability at least $1-\beta$, it holds that
\[
\mathcal{W}\left(
(\psi_{\alpha})_{\#}p\mid\mathcal{E},~ (\psi_{\alpha})_{\#}\hat{p}_n\mid\mathcal{E}
\right)\le \epsilon(N,\beta),
\]
where 
\[
\epsilon(N,\beta):=\left\{
\begin{aligned}
&\left(
\frac{\log(c_1\beta^{-1})}{c_2N}
\right)^{1/\max\{d,3\}},&\quad\text{if }N\ge \frac{\log(c_1\beta^{-1})}{c_2},\\
&\left(
\frac{\log(c_1\beta^{-1})}{c_2N}
\right)^{1/2},&\quad \text{if }N< \frac{\log(c_1\beta^{-1})}{c_2},
\end{aligned}
\right.
\]
and $c_1,c_2$ are some constants depending on $d$ and $\cU$.
Consequently, with probability at least $1-2\beta$, it holds that 
\begin{align*}
&\mathcal{W}\left(
(\psi_{\alpha})_\#(\hat{p}_n\mid\mathcal{E}_{\alpha}),~ (\psi_{\alpha})_\#(\hat{q}_n\mid\mathcal{E}_{\alpha})
\right)\\
\le&\mathcal{W}\left(
(\psi_{\alpha})_{\#}p\mid\mathcal{E},~ (\psi_{\alpha})_{\#}\hat{p}_n\mid\mathcal{E}
\right)
+\mathcal{W}\left(
(\psi_{\alpha})_{\#}q\mid\mathcal{E},~ (\psi_{\alpha})_{\#}\hat{q}_n\mid\mathcal{E}
\right)\\
\le&2\epsilon(N,\beta).
\end{align*}
Applying the union bound for all $\alpha\in\cA$ implies that, the relation
\[
T(\hat{p}_n,\hat{q}_n):=\max_{\alpha\in\cA}~\mathcal{W}\left(
(\psi_{\alpha})_\#(\hat{p}_n\mid\mathcal{E}_{\alpha}),~ (\psi_{\alpha})_\#(\hat{q}_n\mid\mathcal{E}_{\alpha})
\right)
\le 2\epsilon(N,\beta)
\]
holds with probability at least $1-2|\cA|\beta$.
Taking $\beta=\frac{\eta}{2|\cA|}$ implies
\[
\Pr\left\{T(\hat{p}_n,\hat{q}_n)\le 2\epsilon\left(N,\frac{\eta}{2|\cA|}\right)\right\}\ge 1-\eta.
\]
Taking the threshold $\ell_n^{\tII}=2\epsilon\left(N,\frac{\eta}{2|\cA|}\right)$ finishes the first part.

Under the alternative hypothesis $H_1^{\tII}$, the type-II risk can be bounded as
\begin{align*}
\text{Type-II Risk}&=\text{Pr}_{H_1^{\tII}}\left\{
T(\hat{p}_n, \hat{q}_n)< \ell_n^{\tII}
\right\}\\
&=\text{Pr}_{H_1^{\tII}}\left\{
T(p,q) - T(\hat{p}_n, \hat{q}_n)\ge T(p,q) - \ell_n^{\tII}
\right\}\\
&\le \text{Pr}_{H_1^{\tII}}\left\{
\Big|T(p,q) - T(\hat{p}_n, \hat{q}_n)\Big|\ge T(p,q) - \ell_n^{\tII}
\right\}\\
&\le \frac{\mathbb{E}\Big|T(p,q)-
T(\hat{p}_n,\hat{q}_n)\Big|}{\Delta_n^{\tII}},
\end{align*}
where the last relation is based on the Markov inequality and the assumption $T(p,q)>\ell_{n}^{\tII}$, and the expected value
\begin{align*}
\mathbb{E}\Big|T(p,q)-
T(\hat{p}_n,\hat{q}_n)\Big|&\le \max_{\alpha\in\cA}~\mathbb{E}\Big|\mathcal{W}\left(
(\psi_{\alpha})_{\#}p\mid\mathcal{E},~ (\psi_{\alpha})_{\#}q\mid\mathcal{E}
\right)
-
\mathcal{W}\left(
(\psi_{\alpha})_{\#}\hat{p}_n\mid\mathcal{E},~ (\psi_{\alpha})_{\#}\hat{q}_n\mid\mathcal{E}
\right)
\Big|\\
&\le \max_{\alpha\in\cA}~\mathbb{E}\Big|\mathcal{W}\left(
(\psi_{\alpha})_{\#}\hat{p}_n\mid\mathcal{E},~ (\psi_{\alpha})_{\#}p\mid\mathcal{E}
\right)
\Big| + \mathbb{E}\Big|\mathcal{W}\left(
(\psi_{\alpha})_{\#}\hat{q}_n\mid\mathcal{E},~ (\psi_{\alpha})_{\#}q\mid\mathcal{E}
\right)
\Big|.
\end{align*}
Applying Lemma~\ref{Lemma:expectation:wasserstein} on the expected value of empirical Wasserstein distances, it holds that
\[
\mathbb{E}\Big|T(p,q)-
T(\hat{p}_n,\hat{q}_n)\Big|\le c_3N^{-1/\max\{d,2\}}\log(N),
\]
where $c_3$ is a constant depending on $d$ and $\cU$.
The proof of Lemma~\ref{Lemma:test:property:projected:samples} is completed.
\end{proof}

\subsection{Proof of Theorem~\ref{Theorem:testing:agn}}

Combining Lemma~\ref{Lemma:testing:Ang:I}, \ref{Lemma:N:lower:bound}, and \ref{Lemma:test:property:projected:samples} implies the testing properties of $T_{\text{Atlas}}$ in Theorem~\ref{Theorem:testing:agn}.

\begin{proof}[Proof of Theorem~\ref{Theorem:testing:agn}]
We take the confidence level $\eta'=\eta/2$ for testing procedure $T_{\text{Atlas}}^{\tI}$ and $T_{\text{Atlas}}^{\tII}$.
Under the null hypothesis $H_0$, the type-I error can be bounded as
\begin{align*}
&\Pr\bigg\{
\left\{
\|\hat{\cv p}-\hat{\cv q}\|_2\ge \ell_n^{\tI}\right\}~
\bigcup
~\left\{
T(\hat{p}_n, \hat{q}_n) \ge \ell_n^{\tII}
\right\}
\bigg\}\\
=&\Pr\left\{
\|\hat{\cv p}-\hat{\cv q}\|_2\ge \ell_n^{\tI}
\right\}
+
\Pr
\left\{
\|\hat{\cv p}-\hat{\cv q}\|_2< \ell_n^{\tI}\right\}
\Pr\left\{
T(\hat{p}_n, \hat{q}_n) \ge \ell_n^{\tII}~\middle|~\|\hat{\cv p}-\hat{\cv q}\|_2< \ell_n^{\tI}
\right\}\\
\le&\eta' + \Pr\left\{
T(\hat{p}_n, \hat{q}_n) \ge \ell_n^{\tII}~\middle|~\|\hat{\cv p}-\hat{\cv q}\|_2< \ell_n^{\tI}
\right\}\\
\le&2\eta'=\eta.
\end{align*}
Define the event 
\[
\mathcal{G}=\bigg\{
N\ge n\cdot\min_{\alpha\in\cA}~\left(p^{(\alpha)}\land q^{(\alpha)}\right) - \sqrt{2|\cA|\big(2n\ln 2 + n^{3/2}\big)}
\bigg\},
\]
where the constant $N$ is defined in \eqref{Eq:N:expression}.
Under the alternative hypothesis $H_1$, the type-II error can be bounded by considering the following cases.
\begin{itemize}
\item
When $H_0^{\tI}$ does not hold, the type-II error becomes
\begin{align*}
&\text{Pr}_{H_1^{\tI}}
\left\{
\|\hat{\cv p}-\hat{\cv q}\|_2<\ell_n^{\tI}\right\}
\le 2\exp\left(
-\frac{(\Delta_n^{\tI})^2n}{4}
\right).
\end{align*}
\item
On the other hand, when $H_0^{\tI}$ holds while $H_0^{\tII}$ does not, conditioned on the event $\mathcal{G}$, the type-II error becomes
\begin{align*}
&\text{Pr}_{H_0^{\tI}\cap H_1^{\tII}}\bigg\{
\left\{
\|\hat{\cv p}-\hat{\cv q}\|_2<\ell_n^{\tI}\right\}
~\bigcap~
\left\{
T(\hat{p}_n, \hat{q}_n) < \ell_n^{\tII}
\right\}
\big|
\mathcal{G}
\bigg\}\\
=&\text{Pr}_{H_0^{\tI}}
\left\{
\|\hat{\cv p}-\hat{\cv q}\|_2<\ell_n^{\tI}~\middle|~\mathcal{G}\right\}\cdot
\text{Pr}_{H_1^{\tII}}
\left\{
T(\hat{p}_n, \hat{q}_n) < \ell_n^{\tII}~\middle|~ \|\hat{\cv p}-\hat{\cv q}\|_2<\ell_n^{\tI}, \mathcal{G}
\right\}
\\
\le&\text{Pr}_{H_1^{\tII}}
\left\{
T(\hat{p}_n, \hat{q}_n) < \ell_n^{\tII}~\middle|~ \|\hat{\cv p}-\hat{\cv q}\|_2<\ell_n^{\tI},\mathcal{G}
\right\}\\
=&\tilde{O}(n^{-1/\max\{d,2\}})
\end{align*}
where the last relation follows from Lemma~\ref{Lemma:test:property:projected:samples} and the fact that $\mathcal{G}$ holds, with $\tilde{O}(\cdot)$ hiding constants depending on $c,d,\cU$ and linearly depending on $\log n$.
\item
In addition, from Lemma~\ref{Lemma:N:lower:bound} we know that
\[
\Pr\{\mathcal{G}^c\}\le e^{-\sqrt{n}}.
\]
\end{itemize}
As a consequence, the type-II error can be bounded as
\begin{align*}
\text{Type-II Risk}&\le \max\left\{
\text{Pr}_{H_1^{\tI}}
\bigg\{
\|\hat{\cv p}-\hat{\cv q}\|_2<\ell_n^{\tI}\right\},\\
&\qquad\qquad\qquad
\text{Pr}_{H_0^{\tI}\cap H_1^{\tII}}\bigg\{
\left\{
\|\hat{\cv p}-\hat{\cv q}\|_2<\ell_n^{\tI}\right\}
~\bigcap~
\left\{
T(\hat{p}_n, \hat{q}_n) < \ell_n^{\tII}
\right\}
\bigg\}
\bigg\}\\
&\le \max\bigg\{
\text{Pr}_{H_1^{\tI}}
\left\{
\|\hat{\cv p}-\hat{\cv q}\|_2<\ell_n^{\tI}\right\},~\\
&\qquad\qquad\qquad
\text{Pr}_{H_0^{\tI}\cap H_1^{\tII}}\bigg\{
\left\{
\|\hat{\cv p}-\hat{\cv q}\|_2<\ell_n^{\tI}\right\}
~\bigcap~
\left\{
T(\hat{p}_n, \hat{q}_n) < \ell_n^{\tII}
\right\}~
\big|~
\mathcal{G}
\bigg\}+\Pr\{\mathcal{G}^c\}
\bigg\}\\
&\le \max\left\{
2\exp\left(
-\frac{(\Delta_n^{\tI})^2n}{4}
\right), 
\tilde{O}(n^{-1/\max\{d,2\}}) + e^{-\sqrt{n}}
\right\}=\tilde{O}(n^{-1/\max\{d,2\}}).
\end{align*}
with $\tilde{O}(\cdot)$ hiding constants depending on $c,d,\cU$ and linearly depending on $\log n$.
\end{proof}

\section{Proofs of Proposition~\ref{Proposition:discriminative} and Theorem~\ref{Theorem:test:property}}
\label{Thm:test:property:proof}
In this section, we present proofs of Proposition~\ref{Proposition:discriminative} and Theorem~\ref{Theorem:test:property}.
\begin{proof}[Proof of Proposition~\ref{Proposition:discriminative}]
It suffices to show the converse.
Since $d_{\HU}(p,q)=0$, we can assert that 
\[
\mathbb{E}_{p}[f(\cv x)] = \mathbb{E}_{q}[f(\cv x)],\quad \forall f\in\HU.
\]
Or equivalently,
\[
\int_{\mathcal{U}} \big[\mathfrak{h}_p(\cv x) - \mathfrak{h}_q(\cv x)\big]f(\cv x)\diff\cv x=0,\quad \forall f\in\HU.
\]
Since $\mathfrak{h}_p, \mathfrak{h}_q\in \HU$, we can verify that $\overline{f}:=\frac{1}{2}(\mathfrak{h}_p - \mathfrak{h}_q)\in\HU$, which implies
\[
\int_{\mathcal{U}} \big[\mathfrak{h}_p(\cv x) - \mathfrak{h}_q(\cv x)\big]\overline{f}(\cv x)\diff \cv x=
\frac{1}{2}\int_{\mathcal{U}} \big[\mathfrak{h}_p(\cv x) - \mathfrak{h}_q(\cv x)\big]^2\diff\cv x=0.
\]
As a consequence, $\mathfrak{h}_p=\mathfrak{h}_q$ almost surely, which indicates that $p=q$.
\end{proof}

Next, we give a proof of Theorem~\ref{Theorem:test:property}, which is separated into two parts.
In order to specify a threshold to control the type-I risk, we investigate the sample complexity of the empirical H{\"o}lder IPM.
For a fixed threshold, we evaluate the power of obtained test using the Markov inequality.

\paragraph{Type-I Risk of H{\"o}lder IPM Test}
A key technique in this part is McDiarmid's concentration inequality.
\begin{lemma}[McDiarmid's Inequality~{\citep{mcdiarmid_1989}}]
\label{The:Mcdiarmid}
Let $X_1,\ldots,X_n$ be independent random variables, where $X_i$ has the support $\mathcal{X}_i$.
Let $f:\mathcal{X}_1\times\mathcal{X}_2\times\cdots\times\mathcal{X}_n\to\mathbb{R}$ be any function
with the $(c_1,\ldots,c_n)$ bounded difference property, i.e., for $i\in\{1,\ldots,n\}$ and for any $(x_1,\ldots,x_n), (x_1',\ldots,x_n')$ that  differs only in the $i$-th corodinate, we have
\[
|f(x_1,\ldots,x_n) - f(x_1',\ldots,x_n')|\le c_i.
\]
Then for any $t>0$, we have
\[
\text{Pr}\bigg\{
|f(X_1,\ldots,X_n)
-
\mathbb{E}[f(X_1,\ldots,X_n)]
|
\ge t
\bigg\}
\le
2\exp\left(
-\frac{2t^2}{\sum_{i=1}^nc_i^2}
\right).
\]
\end{lemma}
Take $g(\cv X,\cv Y) = d_{\HU}(\hat{p}_n,\hat{q}_n)$ with $\cv X=\{\cv x_i\}_{i=1}^n$ and $\cv Y=\{\cv y_i\}_{i=1}^n$.
For $i=1,\ldots,n$, let $\cv X_{(i)}'$ denote the data set that differs only for the $i$-th component of $\cv X$, and $\cv Y_{(i)}'$ is defined similarly.
Then we can see that 
\begin{align*}
g(\cv X,\cv Y) - g(\cv X_{(i)}', \cv Y)&\le d_{\HU}(\hat{p}_n,\hat{p}_n')\le \frac{1}{n}\sup_{f\in \HU}~|f(\cv x_i) - f(\cv x_i')|\\
&\le \frac{2}{n}~\sup_{f\in \HU}\|f\|_{\infty}\le \frac{2}{n},
\end{align*}
where the last inequality is because, for any function $f\in \HU$, it holds that
\[
\|f\|_{\infty}=\max_{\alpha\in\mathcal{A}}\sup_{\cv x\in U_{\alpha}}~|(f\circ\phi_{\alpha}^{-1})(\phi_{\alpha}(\cv x))|\le \max_{\alpha\in\mathcal{A}}~\|f\circ\phi_{\alpha}^{-1}\|_{\infty}\le 1.
\]
Similarly, 
\begin{align*}
g(\cv X,\cv Y) - g(\cv X, \cv Y'_{(i)})&\le \frac{2}{n}.
\end{align*}
Therefore, applying the Mcdiamard's inequality with $c_i=\frac{1}{n}$ for $i=1,\ldots,2n$ implies 
\[
\text{Pr}\bigg\{
|d_{\HU}(\hat{p}_n,\hat{q}_n) - \mathbb{E}[d_{\HU}(\hat{p}_n,\hat{q}_n)]|\le\epsilon
\bigg\}\le 2\exp\left(
-\frac{n\epsilon^2}{4}
\right).
\]
Moreover, we can upper bound the bias estimation term of the empirical H{\"o}lder IPM using a covering number argument, i.e., using Dudley's Entropy Integral bound in Lemma~\ref{Lemma:entropy} and the covering number bound of the H{\"o}lder function class in Lemma~\ref{Lemma:covering:Holder}.

\begin{lemma}[Dudley’s Entropy Integral Upper Bound~{\citep{Martin19}}]\label{Lemma:entropy}
For a given distribution $\mathfrak{h}$ and function class $\mathcal{F}$ so that $\|f\|_{\infty}\le M$ for any $f\in\mathcal{F}$, we have that
\[
\mathbb{E}[d_{\mathcal{F}}(\mathfrak{h}, \hat{\mathfrak{h}}_n)]
\le 2\inf_{\delta\in(0,M)}~\left(
2\delta + \frac{12}{\sqrt{n}}\int_{\delta}^M
\sqrt{\log \mathcal{N}(\epsilon, \mathcal{F}, \|\cdot\|_{\infty})}\diff\epsilon
\right).
\]
\end{lemma}
\begin{lemma}[Covering Number of H{\"o}lder Function Class]\label{Lemma:covering:Holder}
We have that
\[\mathcal{N}(\epsilon , \HU, \|\cdot\|_{\infty})\le C|\mathcal{A}|\exp\left(
(1/\epsilon)^{\frac{d}{s+\beta}}
\right),
\]
where $C$ is some constant depending on $s,\beta,d$.
\end{lemma}

\begin{proof}[Proof of Lemma~\ref{Lemma:covering:Holder}]
For any $f_1,f_2\in \mathcal{N}(\epsilon , \HU, \|\cdot\|_{\infty})$, we can see that 
\begin{align*}
\|f_1-f_2\|_{\infty}&=\sup_{\cv x\in\mathcal{U}}~|f_1(\cv x) - f_2(\cv x)|\\
&=\max_{\alpha\in\mathcal{A}}~\sup_{\cv x\in U_{\alpha}}~|(f_1\circ\phi_{\alpha}^{-1})(\phi_{\alpha}(\cv x)) - (f_2\circ\phi_{\alpha}^{-1})(\phi_{\alpha}(\cv x))|\\
&=\max_{\alpha\in\mathcal{A}}~\sup_{\cv y\in\phi_{\alpha}(U_{\alpha})}~|(f_1\circ\phi_i^{-1})(\cv y) -(f_2\circ\phi_i^{-1})(\cv y) |\\
&\le \max_{\alpha\in\mathcal{A}}~\|f_1\circ\phi_{\alpha}^{-1} - f_2\circ\phi_{\alpha}^{-1}\|_\infty.
\end{align*}
To construct the cover for $\mathcal{N}(\epsilon , \HU, \|\cdot\|_{\infty})$, we first pick $\{g_{j}\}_{j}$ that forms an $\epsilon$-cover of the H{\"o}lder function class from $\mathbb{R}^d$ to $\mathbb{R}$.
Classical result in \cite{Bracketing07} shows that $|\{g_j\}_{j}|\le C\exp\left(
(1/\epsilon)^{\frac{d}{s+\beta}}
\right).$
Then we can see that $\{g_j\circ \phi_{\alpha}\}_{\alpha,j}$ forms an $\epsilon$-cover of $\HU$.
Moreover, 
\[
\big| 
\{g_j\circ \phi_{\alpha}\}_{\alpha,j}
\big|
\le C|\mathcal{A}|\exp\left(
(1/\epsilon)^{\frac{d}{s+\beta}}
\right).
\]
\end{proof}

In particular, applying Lemma~\ref{Lemma:entropy} on $d_{\HU}(\hat{p}_n,\hat{q}_n)$ gives
\begin{align*}
\mathbb{E}\left[ d_{\HU}(\hat{p}_n,\hat{q}_n)\right]
&\le 
\mathbb{E}\left[d_{\HU}(\hat{p}_n,p)\right] + \mathbb{E}\left[d_{\HU}(\hat{q}_n,q)\right]\\
&\le 4\inf_{\delta\in(0,1)}~\left(
2\delta + \frac{12}{\sqrt{n}}\int_{\delta}^1
\sqrt{\log \mathcal{N}(\epsilon, \HU, \|\cdot\|_{\infty})}\diff\epsilon
\right)\\
&\le c_1\inf_{\delta\in(0,1)}~\left(
2\delta + \frac{12}{\sqrt{n}}\int_{\delta}^1
\sqrt{\log(1/\epsilon)^{\frac{d}{s+\beta}}}\diff\epsilon
\right)\\
&=c_2
n^{-(s+\beta)/d},
\end{align*}
where the infimum in the last step is by taking $\delta=n^{-\frac{s+\beta}{d}}$, and $c_1,c_2$ are constants depending on $s,\beta,d,|\mathcal{A}|$.
Based on the discussion above, we can see that the relation
\[
d_{\HU}(\hat{p}_n,\hat{q}_n)
\le \epsilon + cn^{-(s+\beta)/d}
\]
holds with probability at least $1-2\exp\left(
-\frac{n\epsilon^2}{4}
\right)$, where $c$ is a constant depending on $s,\beta,d,|\mathcal{A}|$.
Taking $\epsilon=2\sqrt{\log(2/\eta)/n}$ completes the proof of the first part.

\paragraph{Type-II Risk of H{\"o}lder IPM Test}\label{Sec:power:Holder}
We can see that the type-II risk can be expressed as
\begin{align*}
\text{Type-II Risk}&=
\text{Pr}_{H_1}\left\{
d_{\HU}(\hat{p}_n,\hat{q}_n)<\ell_{n}
\right\}\\
&=
\text{Pr}_{H_1}\left\{
d_{\HU}(p,q)-
d_{\HU}(\hat{p}_n,\hat{q}_n) \ge d_{\HU}(p,q) - \ell_{n}
\right\}\\
&\le  \text{Pr}_{H_1}\left\{
|d_{\HU}(p,q)-
d_{\HU}(\hat{p}_n,\hat{q}_n)|> d_{\HU}(p,q) - \ell_{n}
\right\}\\
&\le \frac{\mathbb{E}|d_{\HU}(p,q)-
d_{\HU}(\hat{p}_n,\hat{q}_n)|}{d_{\HU}(p,q) - \ell_{n}},
\end{align*}
where the last relation is based on the Markov inequality and the assumption $d_{\HU}(p,q)>\ell_{n}$.
Based on the triangular inequality, we can see that 
\begin{align*}
\mathbb{E}|d_{\HU}(p,q)-
d_{\HU}(\hat{p}_n,\hat{q}_n)|
&\le \mathbb{E}[d_{\HU}(p,\hat{p}_n)]
+
\mathbb{E}[d_{\HU}(q,\hat{q}_n)]\\
&\le c_3n^{-\frac{s+\beta}{d}},
\end{align*}
where $c_3$ is a constant depending on $s,\beta,d,|\mathcal{A}|$.

\section{Proof of Theorem~\ref{Theorem:test:property:app}}
\label{Sec:proof:Theorem:test:property:app}
In this section, we first discuss how to choose hyper-parameters of $\FNN(R,\kappa,L,t,K)$ in Theorem~\ref{Theorem:test:property:app}.
Next, we study the testing properties of the NN IPM test.
\subsection{Choice of hyper-parameters of $\FNN(R,\kappa,L,t,K)$}
Hyper-parameters of $\FNN(R,\kappa,L,t,K)$ determine the size of the neural network class, which should balance the following trade-off.
On the one hand, it should be large enough to approximate the H{\"o}lder function class.
On the other hand, it should be relatively small so that the empirical NN IPM has a sharp sample complexity rate that does not undermine the test efficiency.

We specify hyper-parameters of $\FNN(R,\kappa,L,t,K)$ in Theorem~\ref{Theorem:test:property:app} based on the following proposition, which establishes the error bound for estimating the H{\"o}lder IPM using a NN IPM.

\begin{proposition}[Estimation Error of H{\"o}lder IPM]\label{proposition:Holder:IPM:error}
Let $p,q$ be any two distributions supported on $\cU$ that satisfies Assumption~\ref{Assumption:cU}.
For $\epsilon>0$, specify hyper-parameters of the network class $\FNN(R,\kappa,L,t,K)$ as 
\begin{align*}
R&=1,\quad
\kappa=O(\max\{1, B, \sqrt{d}, \tau^2\}),\quad
L=O\left(\log\frac{1}{\epsilon} + \log D\right),\\
t &= O\left(\epsilon^{-\frac{d}{s+\beta}} + D\right),\quad\text{ and }\quad
K=O\left(\epsilon^{-\frac{d}{s+\beta}}\log\frac{1}{\epsilon} + D\log\frac{1}{\epsilon} + D\log D\right).
\end{align*}
Then it holds that
\begin{equation}
d_{\HU}(p,q)\le d_{\FNN}(p,q) + 2\epsilon.
\label{Eq:upper:bound:holder}
\end{equation}
\end{proposition}
Proposition~\ref{proposition:Holder:IPM:error} reveals that when hyper-parameters of the neural network class are chosen properly, NN IPM serves as an upper bound of the H{\"o}lder IPM with negligible approximation error $\epsilon$.
See the proof of Proposition~\ref{proposition:Holder:IPM:error} in Section~\ref{Sec:proof:Theorem:test:property:app}.
When showing the testing properties of the NN IPM test in Theorem~\ref{Theorem:test:property:app} based on Proposition~\ref{proposition:Holder:IPM:error}, we choose the error term $\epsilon$ to balance the approximation error between H{\"o}lder IPM and NN IPM and the sample complexity of NN IPM (which influences the test efficiency of $\mathcal{T}_{\text{NN}}$). 
\begin{proof}[Proof of Proposition~\ref{proposition:Holder:IPM:error}]
We first study the error bound for approximating the H{\"o}lder IPM using neural networks, which rely on the following universal approximation result.
\begin{proposition}[Theorem~1 in {\cite{chen2020nonparametric}}]\label{proposition:function:approximiation}
For any $f:~\cU\to\mathbb{R}$ belonging to $\HU$, if the hyper-parameters of $\FNN(R,\kappa,L,t,K)$ are properly chosen, the network yields a function $\hat{f}$ satisfying $\|\hat{f} -f\|_{\infty}\le\epsilon$.
Such a neural network has hyper-parameters specified as
\begin{align*}
R&=1,
\kappa=O(\max\{1, B, \sqrt{d}, \tau^2\}),
L=O\left(\log\frac{1}{\epsilon} + \log D\right),\\
t &= O\left(\epsilon^{-\frac{d}{s+\beta}} + D\right),\text{ and }
K=O\left(\epsilon^{-\frac{d}{s+\beta}}\log\frac{1}{\epsilon} + D\log\frac{1}{\epsilon} + D\log D\right)
\end{align*}
\end{proposition}
We specify hyper-parameters of the neural network class the same as in Proposition~\ref{proposition:function:approximiation}.
We denote $f^*$ the optimal discriminator function for $d_{\HU}(p,q)$, and take $\bar{f}\in \FNN(R,\kappa,L,t,K)$ as a function satisfying $\|f - \bar{f}\|_{\infty}\le\epsilon$.
We can see that 
\begin{align*}
&d_{\HU}(p,q) - d_{\FNN}(p,q)
\\=&
\big(\mathbb{E}_{p}[f^*(\cv x)] - \mathbb{E}_{q}[f^*(\cv x)]\big)
-
\sup_{f\in \FNN(R,\kappa,L,t,K)}~\big(\mathbb{E}_{p}[f(\cv x)] - \mathbb{E}_{q}[f(\cv x)]\big)\\
=&\big(\mathbb{E}_{p}[f^*(\cv x)] - \mathbb{E}_{q}[f^*(\cv x)]\big)
+\inf_{f\in \FNN(R,\kappa,L,t,K)}~\big(-\mathbb{E}_{p}[f(\cv x)] + \mathbb{E}_{q}[f(\cv x)]\big)
\\
\le&\big(\mathbb{E}_{p}[f^*(\cv x)] - \mathbb{E}_{q}[f^*(\cv x)]\big) + \big(-\mathbb{E}_{p}[\bar{f}(\cv x)] + \mathbb{E}_{q}[\bar{f}(\cv x)]\big)
\\
=&\mathbb{E}_{p}[(f^* - \bar{f})(\cv x)] + \mathbb{E}_{q}[(\bar{f} - f^* )(\cv x)]\\
\le&2\epsilon,
\end{align*}
where the first inequality is because $\bar{f}$ is a feasible but possibly sub-optimal solution in $\FNN(R,\kappa,L,t,K)$, and the second inequality is because $\|f-\bar{f}\|_{\infty}\le\epsilon$.
The proof of Proposition~\ref{proposition:Holder:IPM:error} is completed.
\end{proof}

Before showing Theorem~\ref{Theorem:test:property:app}, we first study how to pick the approximation error $\epsilon$ so that hyper-parameters of the network class can be determined as indicated in Proposition~\ref{proposition:Holder:IPM:error}. 
\paragraph{Choice of Approximation Error $\epsilon$}
When specifying network hyper-parameters as in Proposition~\ref{proposition:Holder:IPM:error} for a given $\epsilon>0$, the bias estimation error of the H{\"o}lder IPM can be bounded as
\begin{equation}
\begin{aligned}
\mathbb{E}[d_{\HU}(\hat{p}_n,\hat{q}_n)]&\le 2\epsilon + \mathbb{E}[d_{\FNN}(\hat{p}_n,\hat{q}_n)]\\
&\le 
2\epsilon + \mathbb{E}[d_{\FNN}(\hat{p}_n,p)] + \mathbb{E}[d_{\FNN}(\hat{q}_n,q)],
\end{aligned}\label{Eq:Holder:ub}
\end{equation}
where the last relation is by the triangular inequality. 
The bias estimation error for the NN IPM can be upper bounded using a covering number argument.
The following lemma quantifies the covering number of $\FNN$.
\begin{lemma}[Covering Number of Neural Networks, Lemma~7 in {\cite{chen2020statistical}}]\label{Lemma:covering}
We have that
\[
\mathcal{N}(\epsilon , \FNN(R,\kappa,L,t,K), \|\cdot\|_{\infty})\le 
\left(
\frac{2L^2(tB+2)(\kappa t)^{L+1}}{\epsilon}
\right)^K.
\]
\end{lemma}
For simplicity of discussion, we only present the upper bound of $\mathbb{E}[d_{\FNN(R,\kappa,L,t,K)}(\hat{p}_n, p)]$, and the other part can proceed similarly:

\begin{align*}
\mathbb{E}[d_{\FNN}(\hat{p}_n, p)]
&\le 2\inf_{\delta\in(0,1)}~\left(
2\delta + \frac{12}{\sqrt{n}}\int_{\delta}^1
\sqrt{\log \mathcal{N}(\bar{\epsilon} , \FNN(R,\kappa,L,t,K), \|\cdot\|_{\infty})}\diff\bar{\epsilon}
\right)\\
&\le 2\inf_{\delta\in(0,1)}~\left(
2\delta + \frac{12}{\sqrt{n}}\int_{\delta}^1
\sqrt{K\log\frac{2L^2(tB+2)(\kappa t)^{L+1}}{\bar{\epsilon}}}\diff\bar{\epsilon}
\right)\\
&\le c_{1}\left(\frac{1}{n} + \sqrt{\frac{KL\log(nLt)}{n}}\right),
\end{align*}
where the first relation is based on Dudley's Entropy Integral bound in Lemma~\ref{Lemma:entropy}, the second relation is by substituting the covering number upper bound in Lemma~\ref{Lemma:covering}, the last relation is by taking $\delta=1/n$,
and $c_1$ is a constant that relates to $\tau,\sqrt{d},B$.
Then substituting the hyper-parameters $K,L,t$ as specified in Proposition~\ref{proposition:function:approximiation} implies that 
\begin{align*}
\mathbb{E}[d_{\FNN}(\hat{p}_n, p)]
\le c_2\left(\frac{1}{n} + \sqrt{\frac{\epsilon^{-\frac{d}{s+\beta}}(\log(1/\epsilon))^2\log(n\epsilon^{-\frac{d}{s+\beta}})}{n}}\right)
\end{align*}
where $c_2$ is a constant that relates to $\tau,\sqrt{d},B,\sqrt{D}$.

The criteria is to choose $\epsilon$ so that the upper bound in \eqref{Eq:Holder:ub} is as tight as possible:
\begin{align*}
\mathbb{E}[d_{\HU}(\hat{p}_n,\hat{q}_n)]&\le 
\frac{c_2}{n} + 2c_2\inf_{\epsilon}~\left(
\epsilon + \sqrt{\frac{\epsilon^{-\frac{d}{s+\beta}}(\log(1/\epsilon))^2\log(n\epsilon^{-\frac{d}{s+\beta}})}{n}}
\right)\\
&\le c_3n^{-\frac{s+\beta}{2(s+\beta)+d}}\log^2n.
\end{align*}
where in the last relation we take $\epsilon = c\cdot n^{-\frac{s+\beta}{2(s+\beta)+d}}$ to finish the proof,
and $c,c_3$ are constants that relate to $\tau,\sqrt{d},B,\sqrt{D}$.

\subsection{Testing properties of NN IPM test}

In the previous part, we show NN IPM approximates H{\"o}lder IPM with a small estimation error.
As a consequence, we are able to show theoretical guarantees of the NN IPM test in Theorem~\ref{Theorem:test:property:app} based on the study of H{\"o}lder IPM test in Theorem~\ref{Theorem:test:property}.

\begin{proof}[Proof of Theorem~\ref{Theorem:test:property:app}]
The first part can be shown by studying the sample complexity of the empirical NN IPM.
Similar to the proof of Theorem~\ref{Theorem:test:property}, with probability at least $1-2\exp(-\frac{n\delta^2}{4})$, it holds that
\[
d_{\FNN}(\hat{p}_n,\hat{q}_n)\le \delta + \mathbb{E}[d_{\FNN}(\hat{p}_n, \hat{q}_n)]\le \delta + c_1\cdot n^{-\frac{s+\beta}{2(s+\beta)+d}}\log^2n.
\]
Taking $\delta=2\sqrt{\log(2/\eta)/n}$ completes the first part.
In order to show the second part, observe that 
\begin{align*}
\text{Type-II Risk}&=
\text{Pr}_{H_1}\left\{
d_{\FNN(R,\kappa,L,t,K)}(\hat{p}_n,\hat{q}_n)<\ell_{n}
\right\}\\
&\le
\text{Pr}_{H_1}\left\{
d_{\HU}(\hat{p}_n,\hat{q}_n)-2\epsilon<\ell_{n}
\right\}\\
&=
\text{Pr}_{H_1}\left\{
d_{\HU}(\hat{p}_n,\hat{q}_n)<\ell_{n}+2\epsilon
\right\}\\
&\le  \frac{c_3n^{-\frac{s+\beta}{d}}}{d_{\HU}(p,q) - \big[\ell_{n}+2\epsilon\big]},
\end{align*}
where the first inequality is based on Proposition~\ref{proposition:Holder:IPM:error}, and the second inequality is by applying Theorem~\ref{Theorem:test:property} with the threshold being 
 $\ell_{n}+2\epsilon$.
 Note that $\epsilon = c_2\cdot n^{-\frac{s+\beta}{2(s+\beta)+d}}$ with $c_2$ being a constant that relates to $\tau,\sqrt{d},B,\sqrt{D}$.
The proof of Theorem~\ref{Theorem:test:property:app} is completed.
\end{proof}

\section{Conclusion and Additional Discussions}
\label{Sec:conclusion}
This work proposed two-sample tests based on integral probability metrics for high-dimensional samples supported on a low-dimensional manifold.
The theoretical performance of proposed tests with respect to the number of samples $n$ and the structure of the manifold with intrinsic dimension $d$ was investigated.
When an atlas is given, we proposed a two-step test that applies to general distributions, achieving the type-II risk in the order of $n^{-1/\max\{d,2\}}$.
When an atlas is not given, we proposed H{\"o}lder IPM test that applies for distributions with $(s,\beta)$-H{\"o}lder densities, which achieves the type-II risk in the order of $n^{-(s+\beta)/d}$.
To lift the heavy computational burden of evaluating the H{\"o}lder IPM, we approximate the H{\"o}lder function class using neural networks.
We proposed a neural network IPM test and showed that it has the same order of the type-II risk as the H{\"o}lder IPM test.
The performance of proposed tests crucially depends on the intrinsic dimension $d$ instead of the data dimension, suggesting our tests are adaptive to low-dimensional geometric structures.

\paragraph{Most Distinguishable Dimension}
Our proposed two-step test operators by projecting distributions into the local coordinate space of dimension $d$ in each chart of the manifold.
The proposed test may not be optimal in the sense that the direction of projection is not found by exactly optimizing type-I and type-II risk functions.
\citet{liu2020learning} boosts the performance using a more direct approach, by selecting hyper-parameters of the testing function such that the type-II risk is minimized with controlled type-I risk.
Following this spirit, we define the most distinguishable dimension\footnote{We thank the anonymous referee for suggesting this quantity.} as 
\[
d_{\text{dist}}(\eta):=
\underset{d}{\arg\min}~\inf_{f:~\cU\to\mathbb{R}^d, t}~\bigg\{\Pr{}_{H_1}\left(
T_f(\hat{p}_n,\hat{q}_n)\le t
\right):~\Pr{}_{H_0}\left(
T_f(\hat{p}_n,\hat{q}_n)\le t
\right)\ge 1-\eta\bigg\},
\]
where the testing function $T_f$ denotes the Wasserstein distance with projection function $f$:
\[
T_f(p,q):=\mathcal{W}(f_{\#}p, f_{\#}q).
\]
In other words, the most distinguishable dimension $d_{\text{dist}}(\eta)$ selects the optimal dimension of projection such that the type-II risk is minimized with controlled type-I risk.
The dimension $d_{\text{dist}}(\eta)$ does not necessarily equal the intrinsic dimension $d$. Characterizing this quantity is a meaningful research question for future study.

\section*{Data Availability}
No new data were generated or analysed in support of this review.

\section*{Funding}
The work of Jie Wang and Yao Xie is partially supported by NSF CAREER CCF-1650913, NSF DMS-2134037, CMMI-2015787, CMMI-2112533, DMS-1938106, and DMS-1830210.
The work of Tuo Zhao and Wenjing Liao is partially supported by NSF DMS-2012652.
The work of Wenjing Liao is partially supported by NSF CAREER DMS-2145167.

\clearpage
\bibliographystyle{ims}

\end{document}